\documentclass{article}

\usepackage[preprint]{style_file}

\usepackage[utf8]{inputenc} 
\usepackage[T1]{fontenc}    
\usepackage{hyperref}       
\usepackage{url}            
\usepackage{booktabs}       
\usepackage{amsfonts}       
\usepackage{nicefrac}       
\usepackage{microtype}      
\usepackage{xcolor}         

\usepackage{amsmath}
\usepackage{amssymb}
\usepackage{mathtools}
\usepackage{amsthm}

\usepackage{graphicx}

\usepackage{bm}
\usepackage{bbm}
\usepackage{enumerate}
\usepackage{comment}
\usepackage{nicefrac}  
\usepackage{here}

\theoremstyle{plain}
\newtheorem{theorem}{Theorem}[section]

\newtheorem{lemma}[theorem]{Lemma}
\newtheorem{corollary}[theorem]{Corollary}
\theoremstyle{definition}

\theoremstyle{remark}

\makeatletter
\newcommand*\rel@kern[1]{\kern#1\dimexpr\macc@kerna}
\newcommand*\widebar[1]{%
  \begingroup
  \def\mathaccent##1##2{%
    \rel@kern{0.8}%
    \overline{\rel@kern{-0.8}\macc@nucleus\rel@kern{0.2}}%
    \rel@kern{-0.2}%
  }%
  \macc@depth\@ne
  \let\math@bgroup\@empty \let\math@egroup\macc@set@skewchar
  \mathsurround\z@ \frozen@everymath{\mathgroup\macc@group\relax}%
  \macc@set@skewchar\relax
  \let\mathaccentV\macc@nested@a
  \macc@nested@a\relax111{#1}%
  \endgroup
}
\makeatother

\newcommand{\argmax}{\mathop{\rm arg~max}\limits}
\newcommand{\argmin}{\mathop{\rm arg~min}\limits}

\title{Learning from Hard Labels with \\ Additional Supervision on Non-Hard-Labeled Classes}

\author{%
  Kosuke Sugiyama \\
  Waseda University\\
  3-4-1 Okubo, Shinjuku, Tokyo 169-8555, Japan \\
  \texttt{kohsuke0322@asagi.waseda.jp} \\
  \And
  Masato Uchida \\
  Waseda University \\
  3-4-1 Okubo, Shinjuku, Tokyo 169-8555, Japan \\
  \texttt{m.uchida@waseda.jp} \\
}

\begin{document}

\maketitle

\begin{abstract}
In scenarios where training data is limited due to observation costs or data scarcity, enriching the label information associated with each instance becomes crucial for building high-accuracy classification models.
In such contexts, it is often feasible to obtain not only hard labels but also {\it additional supervision}, such as the confidences for the hard labels.
This setting naturally raises fundamental questions: {\it What kinds of additional supervision are intrinsically beneficial?} And {\it how do they contribute to improved generalization performance?}
To address these questions, we propose a theoretical framework that treats both hard labels and additional supervision as probability distributions, and constructs soft labels through their affine combination.
Our theoretical analysis reveals that the essential component of additional supervision is not the confidence score of the assigned hard label, but rather the information of the distribution over the non-hard-labeled classes.
Moreover, we demonstrate that the additional supervision and the mixing coefficient contribute to the refinement of soft labels in complementary roles. Intuitively, in the probability simplex, the additional supervision determines the direction in which the deterministic distribution representing the hard label should be adjusted toward the true label distribution, while the mixing coefficient controls the step size along that direction.
Through generalization error analysis, we theoretically characterize how the additional supervision and its mixing coefficient affect both the convergence rate and asymptotic value of the error bound.
Finally, we experimentally demonstrate that, based on our theory, designing additional supervision can lead to improved classification accuracy, even when utilized in a simple manner.
\end{abstract}

\section{Introduction}
\label{sec:introduction}

Building high-accuracy classification models typically requires large-scale training datasets that contain sufficient information about the true data distribution.
However, in many real-world scenarios, collecting a sufficient number of instances is fundamentally challenging due to factors such as the rarity of the target phenomenon, high observation costs, or time constraints.
Under such constraints, increasing the number of instances may be infeasible, making it crucial to enhance the informational content of each supervised label to compensate for data scarcity.

While supervised labels are typically provided as hard labels that indicate a single correct class, the true class assignment is often inherently probabilistic.
This can occurs, for example, when different annotators assign different labels to the same instance, or when the label reflects the outcome of a stochastic event.
From this perspective, soft labels that reflect the true label distribution provide a more informative alternative to deterministic hard labels.
A common approach is to aggregate labels from multiple annotators to estimate true label distribution \cite{sheng2008get,rodrigues2018deep,peterson_2019_ICCV,battleday2020capturing,uma2020case,fornaciari2021beyond,davani2022dealing,collins2022eliciting,wu2023don}.
However, such methods are not applicable when it is difficult to recruit multiple annotators or when labels cannot be re-observed, such as in the case of logged data involving one-time, non-repeatable events.

In contrast, even in such settings, it is often possible to access {\it additional supervision} that conveys partial information about the label distribution, such as confidence scores associated with hard labels or probability estimates over alternative classes.
For example, domain experts may assign confidence scores to hard labels based on prior knowledge.
Leveraging such additional supervision can enable the construction of effective soft labels without the need for multiple annotators, even in cases where labels cannot be re-observed. 
 To fully realize this potential, it is essential to address the following two fundamental questions:

\textbf{RQ1: What types of additional supervision are effective in improving the quality of soft labels?}
Prior work has explored strategies for obtaining soft labels from annotators, such as requesting the probability assigned to the most probable class or to one of the least probable classes~\cite{collins2022eliciting,wu2023don}. 
While these studies provide empirical evidence for the usefulness of such signals, the theoretical relationship between the type of additional supervision and the quality of the resulting soft label, in terms of their closeness to the true label distribution, has not been systematically investigated.

\textbf{RQ2: How does improving label quality affect the generalization performance of classification models?}
Theoretical studies on learning from weakly supervised labels have examined various forms of label imperfection \cite{cour2011learning,ishida2017cll,tang2025confidence,cao2021sconf,wang2023confdiff}. 
 Theoretical analyses also exist for learning with soft labels that correct or refine hard labels \cite{wang2019theoretical}.
However, there is still limited theoretical understanding of how the quality of soft labels affects the generalization performance of classification models.

Addressing these research questions requires a theoretical framework that integrates hard labels and additional supervision to construct soft labels. 
In this work, we propose such a framework and provide its theoretical analysis of its properties.
We begin by formulating the framework as a problem of estimating a probability distribution in Section~\ref{sec:formulation}.
Various forms of additional supervision associated with each instance $\bm{x}$ can be represented as probability distributions, as illustrated in Figure~\ref{fig:eg_additional_supervision}.
Accordingly, we denote this additional supervision by $p_{\mathrm{A}}(y|\bm{x})$.
We then consider a natural way to construct a soft label $p_{\lambda}(y|\bm{x})$ as a mixture distribution obtained by an affine combination of two probability distributions:
The first is the hard label, represented as a deterministic distribution that assigns probability one to a single class. 
The second is $p_{\mathrm{A}}(y|\bm{x})$, the distribution corresponding to the additional supervision. 
These are combined using a mixing coefficient $\lambda(\bm{x})$.
The optimal mixing coefficient $\lambda(\bm{x})$ is then obtained by minimizing the Kullback-Leibler (KL) divergence between the true label distribution and $p_{\lambda}(y|\bm{x})$.

\begin{figure*}[t]
    \centering
    \includegraphics[width=0.98\linewidth]{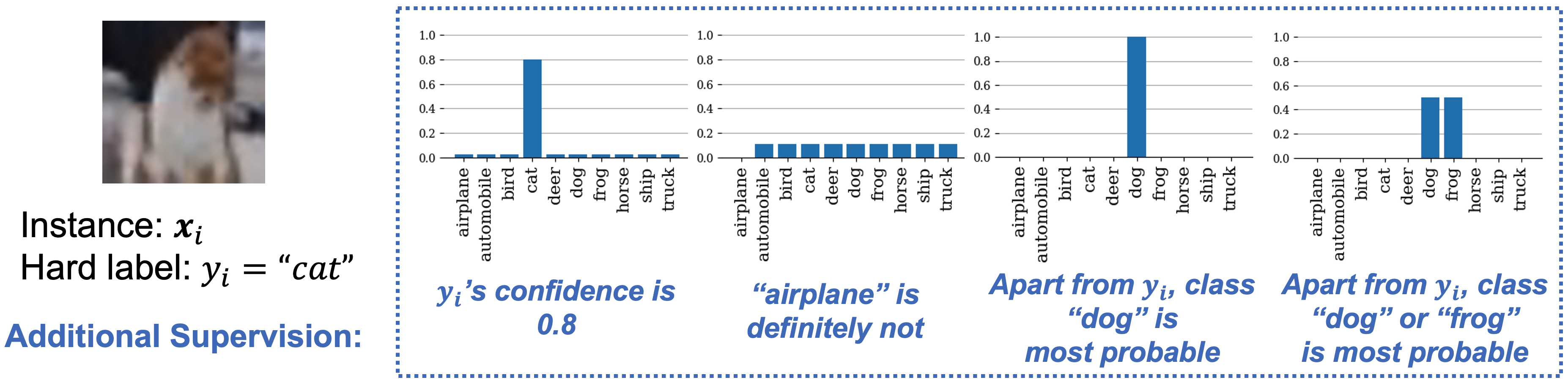}
    \caption{Examples of additional supervision in CIFAR-10 dataset \cite{cifar10}.}
    \label{fig:eg_additional_supervision}
\end{figure*}

\textbf{Answer for RQ1.}
Under the proposed formulation, we show that $p_{\mathrm{A}}(y|\bm{x})$ and the mixing coefficient $\lambda(\bm{x})$ make independent contributions to improving soft labels and play fundamentally different roles in Section~\ref{subsec:interpret_formulation}.
We first show that the proposed objective function admits a bias-variance decomposition using information geometry.
In this decomposition, the bias term quantifies the discrepancy between the true label distribution and the optimal mixture distribution $p_{\mathrm{opt}}(y|\bm{x})$, while the variance term reflects the discrepancy between $p_{\mathrm{opt}}(y|\bm{x})$ and the estimated mixture distribution $p_{\lambda}(y|\bm{x})$.
A detailed analysis of this decomposition reveals that the bias term captures how accurately $p_{\mathrm{A}}(y|\bm{x})$ represents the class proportions excluding the hard-labeled class, and is independent of the mixing coefficient $\lambda(\bm{x})$.
This insight leads to a key conclusion: the additional supervision $p_{\mathrm{A}}(y|\bm{x})$ does not need to contain any information about the hard-labeled class; it is sufficient for it to encode information only about the non-hard-labeled classes.
In contrast, we show that the variance term quantifies how accurately $p_{\lambda}(y|\bm{x})$ estimates the true probability of the hard-labeled class, and that minimizing this term is not affected by the choice of $p_{\mathrm{A}}(y|\bm{x})$.
This analysis further shows that optimizing $\lambda(\bm{x})$ requires no knowledge of the additional supervision and is equivalent to estimating the true probability of the hard-labeled class.

\textbf{Answer for RQ2.}
We further provide a theoretical investigation for how the design of $p_{\mathrm{A}}(y|\bm{x})$ and $\lambda(\bm{x})$ affects classification performance through a generalization error analysis in a finite-sample setting in Section~\ref{subsec:error_analysis}.
We first derive an error bound that explicitly characterizes how the generalization error is affected by the discrepancy between the soft label $p_{\lambda}(y|\bm{x})$ and the true label distribution.
The derived error bound quantitatively characterizes how the design of $p_{\mathrm{A}}(y|\bm{x})$ and $\lambda(\bm{x})$, which determine $p_{\lambda}(y|\bm{x})$, influences the convergence rate and asymptotic value of the error bound.
Specifically, we show that when $p_{\lambda}(y|\bm{x})$ significantly deviates from $p_{*}(y|\bm{x})$, the error bound is dominated by the slower-converging term of order $\mathcal{O}_p(1/n^{1/4})$, rather than the faster-converging term of order $\mathcal{O}_p(1/n^{1/2})$.
This result provides theoretical support for the claim that improving $p_{\mathrm{A}}(y|\bm{x})$ and $\lambda(\bm{x})$ contributes to reducing the generalization error of classification models.

Based on our theoretical insights, we experimentally validate that even simple designs and usage of $p_{\mathrm{A}}(y|\bm{x})$ can improve generalization performance in Section~\ref{sec:experiments}.
Specifically, we demonstrate that selecting the most probable Top-1 or Top-2 class among the non-hard-labeled classes as $p_{\mathrm{A}}(y|\bm{x})$ is effective, even when $\lambda(\bm{x})$ is set as a constant, representing the simplest possible implementation.

\section{Related Work}
\label{seq:related_work}

\subsection{Supervised Learning}
\label{subsec:related_supervised}
In this paper, we consider the supervised learning framework for classification tasks based on Empirical Risk Minimization (ERM) \cite{mohri2018foundations,shalev2014understanding}.
ERM assumes a training dataset where each instance is annotated with a single hard label.
Let the input space be denoted by $\mathcal{X} \subseteq \mathbb{R}^d$ and the label space by $\mathcal{Y} \subset \mathbb{R}$, where $d$ is the dimensionality of the input.
Let $(\bm{X}, Y)$ be random variables over $\mathcal{X} \times \mathcal{Y}$ representing instances and their hard labels. 
We assume that i.i.d.\ samples $(\bm{x}, y)$ are drawn from the true joint distribution $p_{*}(\bm{x}, y)$.
The goal in ERM is to learn a classification model $f: \mathcal{X} \rightarrow \mathcal{Y} \in \mathcal{F}$ that minimizes the expected risk $R_{l}(f) = \mathbb{E}_{p_*(\bm{x}, y)}[l(f(\bm{X}), Y)]$, 
where $l: \mathbb{R} \times \mathbb{R} \rightarrow \mathbb{R}+$ is a loss function, and $\mathcal{F}$ is a hypothesis class for classification models.
In practice, ERM approximates this expected risk using a finite training sample $S := \{(\bm{x}_i, y_i)\}_{i=1}^n$ with $n \in \mathbb{N}+$, leading to the empirical risk $\hat{R}_{l,S}(f) = \mathbb{E}_{p_S(\bm{x}, y)}[l(f(\bm{X}), Y)] = \frac{1}{n}\sum_{i=1}^n l(f(\bm{x}_i), y_i)$, which is minimized to learn $f$.
Here, $p_S$ denotes the empirical distribution defined by $S$, and $y_i$ is the hard label for instance $\bm{x}_i$.

\subsection{Constructing Soft Labels by Multiple Annotators}
\label{subsec:related_annotator}

In some tasks, different annotators may assign different labels to the same instance.
To handle this uncertainty, various methods have been studied that collect hard labels from multiple annotators and combine them to form soft labels that reflect annotation ambiguity \cite{sheng2008get,rodrigues2018deep,peterson_2019_ICCV,battleday2020capturing,uma2020case,fornaciari2021beyond,davani2022dealing}.
These methods can be viewed as using repeated human annotation to approximate the underlying label distribution.
However, this approximation often incurs substantial annotation costs. 

To reduce annotation costs, Collins et al. (2022) and Wu et al. (2023) proposed collecting richer label information from a small number of annotators \cite{collins2022eliciting,wu2023don}. 
Each annotator provides the most probable class with its probability, the next most probable class, and a class assigned zero probability.
However, accurately assigning probability values is often challenging, and multiple annotators may still be required.

Furthermore, all of these methods assume manual annotation as a prerequisite, rendering them inapplicable to scenarios in which labels are passively obtained where repeated observation of the same instance is infeasible.
In contrast, even in such cases, it may be possible to incorporate additional supervision about the label distribution based on domain knowledge.
In this study, we propose a practical framework for generating soft labels by augmenting each hard-labeled instance with a single piece of additional information. 
This allows us to model label distributions effectively without relying on multiple annotators.

\subsection{Constructing Soft Labels Using only the Train Dataset}
\label{subsec:related_using_data}

Several approaches have been proposed to refine supervision solely from existing hard-labeled training data.
Representative methods include label enhancement (LE) \cite{xu2020variational,xu2021label,qianghai2023generalized,zheng2023label}, knowledge distillation (KD) \cite{hinton2015distilling,gou2021knowledge,furlanello2018born}, and label smoothing (LS) \cite{Szegedy_2016_CVPR}.
LE constructs more plausible soft labels by comparing the hard label assigned to a sample with those assigned to its neighbors in the feature space \cite{xu2020variational,xu2021label,qianghai2023generalized,zheng2023label}.
This framework has also been extended to weakly supervised learning settings \cite{xu2019partial}.
KD was originally proposed as a model compression technique \cite{hinton2015distilling,gou2021knowledge}.
In KD, a teacher model is first trained using hard labels, and its output distribution is then used as soft labels to train a student model.
This procedure, known as self-distillation, has been shown to sometimes allow the student to outperform the teacher \cite{furlanello2018born}.
LS constructs soft labels by perturbing the hard labels with uniform noise, which has been shown to improve generalization and calibration \cite{Szegedy_2016_CVPR,muller2019does}.

Despite their effectiveness, these methods have limitations.
LE and KD may be less effective when the amount of data is limited. 
Moreover, LS does not account for label distributions and therefore provides limited improvement in label quality.
In this work, we propose a framework that overcomes these limitations by introducing additional supervision to enable more effective soft label generation.

\section{Formulation}
\label{sec:formulation}

In this section, we formulate the refinement of supervision using additional supervision as a problem of estimating a probability distribution.
Two key considerations in this formulation are: (i) how the additional supervision is represented, and (ii) how it is utilized.
Simple examples of such additional supervision include ``the class most probable to be assigned other than the hard label'' or ``classes with zero probability of being the hard label.''
As illustrated in Figure~\ref{fig:eg_additional_supervision}, such information can be naturally represented as a probability distribution over $y$ for each instance $\bm{x}_i$.
Based on this observation, we denote the additional supervision for each $\bm{x}_i$ by a probability distribution $p_{\mathrm{A}}(y|\bm{x}_i)$.

Given this setup, the problem of refining label can be viewed as constructing a soft label that approximates the true label distribution $p_*(y|\bm{x}_i)$, by integrating the hard label represented as a deterministic distribution $p_{S}(y|\bm{x}_i)=\mathbbm{1}_{[y=y_i]}$ and $p_{\mathrm{A}}(y|\bm{x}_i)$.
In this work, we consider performing this integration via an affine combination, one of the most natural and interpretable approaches to combining distributions.
Specifically, we define the resulting soft label $p_{\lambda}(y|\bm{x}_i)$ as follows:
\begin{align}
    p_{\lambda}(y|\bm{x}_i) := \lambda(\bm{x}_i)p_{S}(y|\bm{x}_i) + (1- \lambda(\bm{x}_i))p_{\mathrm{A}}(y|\bm{x}_i), ~~ \forall i \in [n],
\label{eq:linear_comb}
\end{align}
where the mixing coefficient $\lambda: \mathcal{X} \rightarrow \mathbb{R}$ controls the degree of integration for each instance.
We define $\Delta_{\mathcal{Y}}$ as the set of probability distributions over $\mathcal{Y}$, and consider only those $\lambda$ for which $p_{\lambda}(y|\bm{x}_i) \in \Delta_{\mathcal{Y}}$ holds for any $i \in [n]$.
Here, the value of $\lambda(\bm{x}_i)$ is not restricted to the range $[0, 1]$.

Ideally, the mixing coefficient $\lambda(\bm{x}_i)$ should be optimized so that  $p_{\lambda}(y|\bm{x}_i)$ closely approximates the true label distribution $p_{*}(y|\bm{x}_i)$.
To quantify the quality of $p_{\lambda}(y|\bm{x}_i)$, we adopt the KL divergence between $p_*(y|\bm{x}_i)$ and $p_{\lambda}(y|\bm{x}_i)$, denoted by $D_{\mathrm{KL}}(p_*(Y|\bm{x}_i) || p_{\lambda}(Y|\bm{x}_i)) := \mathbb{E}_{p_*(y|\bm{x}_i)}\big[\log \frac{p_*(Y|\bm{x}_i)}{p_{\lambda}(Y|\bm{x}_i)}\big]$.
Then, the optimization problem for $\lambda(\bm{x}_i)$ is formulated as follows:
\begin{align}
\textstyle \min_{\lambda(\bm{x}_i)} D_{\mathrm{KL}}(p_*(Y|\bm{x}_i) || p_{\lambda}(Y|\bm{x}_i)), ~~~ \forall i \in [n].
\label{eq:mixture_optimization}
\end{align}
In practice, since $p_*(y|\bm{x}_i)$ is unobservable, the optimization in Eq.~\eqref{eq:mixture_optimization} cannot be executed.
While the KL divergence in Eq.~\eqref{eq:mixture_optimization} cannot be explicitly computed, it serves as a theoretical basis for the analysis presented in Section~\ref{subsec:interpret_formulation}.
Through this analysis, we derive both the desirable properties that $p_{\mathrm{A}}(y|\bm{x}_i)$ should satisfy and a principled strategy for optimizing $\lambda(\bm{x}_i)$ without accessing to $p_*(y|\bm{x}_i)$.
Furthermore, Section \ref{subsec:error_analysis} specifically demonstrates how the objective function in Eq.~\eqref{eq:mixture_optimization}, representing the quality of $p_{\lambda}(y|\bm{x})$, influences both the convergence rate and asymptotic value of the error bound for the classification model trained using $p_{\lambda}(y|\bm{x})$.

\section{Theoretical Analysis}
\label{sec:theoretical_analysis}

In our formulation, refining supervision requires both a better form of additional supervision $p_{\mathrm{A}}(y|\bm{x})$ and an appropriately chosen mixing coefficient $\lambda(\bm{x})$.
This leads to two central questions: How do improvements in $p_{\mathrm{A}}(y|\bm{x})$ and $\lambda(\bm{x})$ influence the refinement of labels?
And, how does such refinement impact the generalization performance of a classification model?
This section provides theoretical answers to these questions in Sections~\ref{subsec:interpret_formulation} and~\ref{subsec:error_analysis}, respectively.

\subsection{Theoretical Interpretation for the Formulation}
\label{subsec:interpret_formulation}

In this section, through the analysis of the objective function in Eq.~\eqref{eq:mixture_optimization}, we clarify the desirable properties of $p_{\mathrm{A}}(y|\bm{x}_i)$ and establish a policy for optimizing $\lambda(\bm{x}_i)$.
Our analysis lead to two key insights:
First, for $p_{\mathrm{A}}(y|\bm{x}_i)$, it is not necessary to encode information about the hard-labeled class. 
Instead, it is fundamentally important to accurately represent the occurrence proportions of the non-hard-labeled classes.
Second, the optimization of $\lambda(\bm{x}_i)$ is equivalent to estimating the probability of the hard-labeled class, and this process does not require any reference to $p_{\mathrm{A}}(y|\bm{x}_i)$.

These findings are derived by performing a bias-variance decomposition of the objective function in Eq.~\eqref{eq:mixture_optimization} and analyzing the resulting terms.
This bias–variance decomposition can be derived using information geometry and is expressed as follows.
The proof is provided in Appendix~\ref{apdx_subsec:proof_thm:bias_var_decomposition}.
\begin{theorem}
    For any $i \in [n]$ and any $\lambda$ such that $p_{\lambda}(y|\bm{x}_i) \in \Delta_{\mathcal{Y}}$, the following holds:
    \begin{align}
         D_{\mathrm{KL}}(p_*(Y|\bm{x}_i)|| p_{\lambda}(Y|\bm{x}_i) )
        = \underbrace{D_{\mathrm{KL}}(p_*(Y|\bm{x}_i)|| p_{\mathrm{opt}}(Y|\bm{x}_i) )}_{=:\mathrm{Bias}}
            + \underbrace{D_{\mathrm{KL}}(p_{\mathrm{opt}}(Y|\bm{x}_i)||  p_{\lambda}(Y|\bm{x}_i))}_{=: \mathrm{Variance}},
    \label{eq:bias_var_decomposition}
    \end{align}
    where, $p_{\mathrm{opt}}(y|\bm{x}_i) := p_{\lambda=\lambda^*}(y|\bm{x}_i), ~~ \lambda^* := \arg\min_{\lambda}D_{\mathrm{KL}}(p_*(Y|\bm{x}_i)|| p_{\lambda}(Y|\bm{x}_i) )$.
\label{thm:bias_var_decomposition}
\end{theorem}
The bias term quantifies the discrepancy between the true label distribution $p_{*}(y|\bm{x}_i)$ and the optimal mixture distribution $p_{\mathrm{opt}}(y|\bm{x}_i)$, which is defined using the optimal mixing coefficient $\lambda^*$.
In contrast, the variance term measures the difference between $p_{\mathrm{opt}}(y|\bm{x}_i)$ and a particular mixture distribution $p_{\lambda}(y|\bm{x}_i)$.

By analyzing the bias and variance terms in Theorem~\ref{thm:bias_var_decomposition}, we further clarify their relationships with $p_{\mathrm{A}}(y|\bm{x}_i)$ and $\lambda(\bm{x}_i)$.
To support this analysis, we first present the following key equations, with their proofs are provided in Appendix~\ref{apdx_subsec:proof_lem:dist_equations}.
\begin{lemma}
    For any $i \in [n]$ and any $\lambda$ such that $p_{\lambda}(y|\bm{x}_i) \in \Delta_{\mathcal{Y}}$, the following holds:
    \begin{align}
        &\textstyle \frac{1}{p_{\mathrm{A}}(Y\neq y_i | \bm{x}_i)} p_{\mathrm{A}}(y'| \bm{x}_i) = \frac{1}{p_{\lambda}(Y\neq y_i | \bm{x}_i)} p_{\lambda}(y'| \bm{x}_i), ~~~~\forall y' \in \mathcal{Y}\setminus \{y_i\}, \label{eq:equal_lam_add} \\
        &p_*(y_i|\bm{x}_i)=p_{\mathrm{opt}}(y_i|\bm{x}_i),
    \label{eq:equal_star_opt}
    \end{align}
    where $p_{\mathrm{A}}(Y\neq y_i|\bm{x}_i) := \sum_{y \in \mathcal{Y}\setminus\{y_i\}}p_{\mathrm{A}}(y|\bm{x}_i)$ and $p_{\lambda}(Y\neq y_i|\bm{x}_i) := \sum_{y \in \mathcal{Y}\setminus\{y_i\}}p_{\lambda}(y|\bm{x}_i)$.
\label{lem:dist_equations}
\end{lemma}
Equation~\eqref{eq:equal_lam_add} shows that the class proportions for all classes other than $y_i$ are always matched between $p_{\mathrm{A}}(y|\bm{x}_i)$ and $p_{\lambda}(y| \bm{x}_i)$.
Meanwhile, Eq.~\eqref{eq:equal_star_opt} demonstrates that, by selecting the optimal mixing coefficient, the probability assigned to $y_i$ in $p_{\lambda}(y|\bm{x}_i)$ can be exactly aligned with that of $p_{*}(y|\bm{x}_i)$.

Using Lemma~\ref{lem:dist_equations}, we obtain the following theorem, whose proof is provided in Appendix~\ref{apdx_subsec:proof_thm:transform_bias_variance}.
\begin{theorem}
    For any $i \in [n]$ and any $\lambda$ such that $p_{\lambda}(y|\bm{x}_i) \in \Delta_{\mathcal{Y}}$, the following holds:
    \begin{align}
    \begin{split}
        \mathrm{Bias} 
        &= (1 -p_*(y_i | \bm{x}_i)) D_{\mathrm{KL}}(p_{*, \neq y_i}(Y|\bm{x}_i) || p_{\mathrm{A}, \neq y_i}(Y|\bm{x}_i)),
    \end{split}
    \label{eq:bias_representation}
    \end{align}
    \begin{align}
    \begin{split}
        \mathrm{Variance}
        & \textstyle =p_{*}(y_i|\bm{x}_i) \log \frac{p_{*}(y_i|\bm{x}_i)}{p_{\lambda}(y_i|\bm{x}_i)}  
            + (1- p_{*}(y_i|\bm{x}_i)) \log \frac{1 - p_{*}(y_i|\bm{x}_i)}{1- p_{\lambda}(y_i|\bm{x}_i)},
    \end{split}
    \label{eq:variance_representation}
    \end{align}   

    where $p_{*, \neq y_i}(y|\bm{x}_i) := \begin{cases} \frac{p_*(y|\bm{x}_i)}{p_*(Y\neq y_i|\bm{x}_i)} & \mathrm{if } y \neq y_i \\ 0 & \mathrm{if } y = y_i\end{cases}$ and $p_{\mathrm{A}, \neq y_i}(y|\bm{x}_i) := \begin{cases} \frac{p_{\mathrm{A}}(y|\bm{x}_i)}{p_{\mathrm{A}}(Y\neq y_i|\bm{x}_i)} & \mathrm{if } y \neq y_i \\ 0 & \mathrm{if } y = y_i\end{cases}$.
\label{thm:transform_bias_variance}
\end{theorem}

Theorem~\ref{thm:transform_bias_variance} provides a concrete interpretation of the quantities measured by the bias and variance terms.
Specifically, Eq.~\eqref{eq:bias_representation} shows that the bias term quantifies how accurately $p_{\mathrm{A}}(y|\bm{x}_i)$ represents the class proportions excluding the hard labeled class $y_i$.
Meanwhile, Eq.~\eqref{eq:variance_representation} indicates that the variance term measures how accurately $p_{\lambda}(y|\bm{x}_i)$ estimates the true probability $p_*(y_i|\bm{x}_i)$ of $y_i$.

Theorem~\ref{thm:transform_bias_variance} first establishes that $p_{\mathrm{A}}(y|\bm{x}_i)$ and $\lambda(\bm{x}_i)$ independently contribute to refining $p_{\lambda}(y|\bm{x}_i)$.
Equation~\eqref{eq:bias_representation} makes it clear that reducing the bias term depends solely on $p_{\mathrm{A}}(y|\bm{x}_i)$ and is independent of $\lambda(\bm{x}_i)$.
Conversely, Eq.~\eqref{eq:variance_representation} shows that reducing the variance term requires setting $p_{\lambda}(y_i|\bm{x}_i)$ close to $p_{*}(y_i|\bm{x}_i)$.
According to Eq.~\eqref{eq:equal_star_opt} in Lemma~\ref{lem:dist_equations}, this alignment can be achieved solely by adjusting $\lambda(\bm{x}_i)$, resulting in $p_{\lambda}(y_i|\bm{x}_i) = p_{*}(y_i|\bm{x}_i)$ and reducing the variance term to zero.
Therefore, only $\lambda(\bm{x}_i)$ has a substantive effect on reducing the variance term.
Taken together, $p_{\mathrm{A}}(y|\bm{x}_i)$ and $\lambda(\bm{x}_i)$ independently contribute to reducing the bias and variance terms, respectively.

Furthermore, Theorem~\ref{thm:transform_bias_variance} reveals two key insights: first, that only the information about non-hard-labeled classes is required as $p_{\mathrm{A}}(y|\bm{x}_i)$, and second, that optimizing $\lambda(\bm{x}_i)$ is equivalent to estimating the true probability of the hard label.
As shown earlier, improving $p_{\mathrm{A}}(y|\bm{x}_i)$ reduces only the bias term and has no effect on the variance term.
Combined with Eq.~\eqref{eq:bias_representation}, this implies that $p_{\mathrm{A}}(y|\bm{x}_i)$ does not need to contain any information such as confidence scores about the hard labeled class $y_i$.
Instead, it only needs to represent the class proportions over the non-hard-labeled classes accurately.
This result establishes a theoretically grounded requirement for additional supervision and offers a concrete guideline for practical implementation.
Furthermore, as discussed above, since $\lambda(\bm{x}_i)$ should be chosen to satisfy $p_{\lambda}(y_i|\bm{x}_i) = p_{*}(y_i|\bm{x}_i)$, optimizing $\lambda(\bm{x}_i)$ is equivalent to estimating the true probability for $y_i$.
This equivalence indicates that the optimization of $\lambda(\bm{x}_i)$ does not require any information from $p_{\mathrm{A}}(y|\bm{x}_i)$ and can be based solely on the probability of $y_i$. 
This provides a useful guideline for choosing $\lambda(\bm{x}_i)$ when $p_{*}(y_i|\bm{x}_i)$ is not available.

This result can also be interpreted from a geometric perspective.
Figure~\ref{fig:geometric_interpretation} depicts the probability simplex where $\mathcal{Y}=\{y_1,y_2,y_3\}$ and visually represents the findings described above.
Using this figure, our framework can be interpreted as follows:
First, the specification of $p_{\mathrm{A}}(y|\bm{x}_i)$ constrains $\{p_{\lambda}(y|\bm{x}_i) | \lambda(\bm{x}_i) \in \mathbb{R} \land p_{\lambda}(y|\bm{x}_i) \in \Delta_{\mathcal{Y}} \}$ to lie along a single blue line.
Then, choosing $\lambda(\bm{x}_i)$ corresponds to selecting the point on this line that is closest to $p_{*}(y|\bm{x}_i)$.
In other words, on the simplex, $p_{\mathrm{A}}(y|\bm{x}_i)$ specifies the direction in which $p_{S}(y|\bm{x}_i)$ moves toward $p_{*}(y|\bm{x}_i)$, while $\lambda(\bm{x}_i)$ determines the magnitude of that movement.

\begin{figure*}[t]
    \centering
    \includegraphics[width=0.90\linewidth]{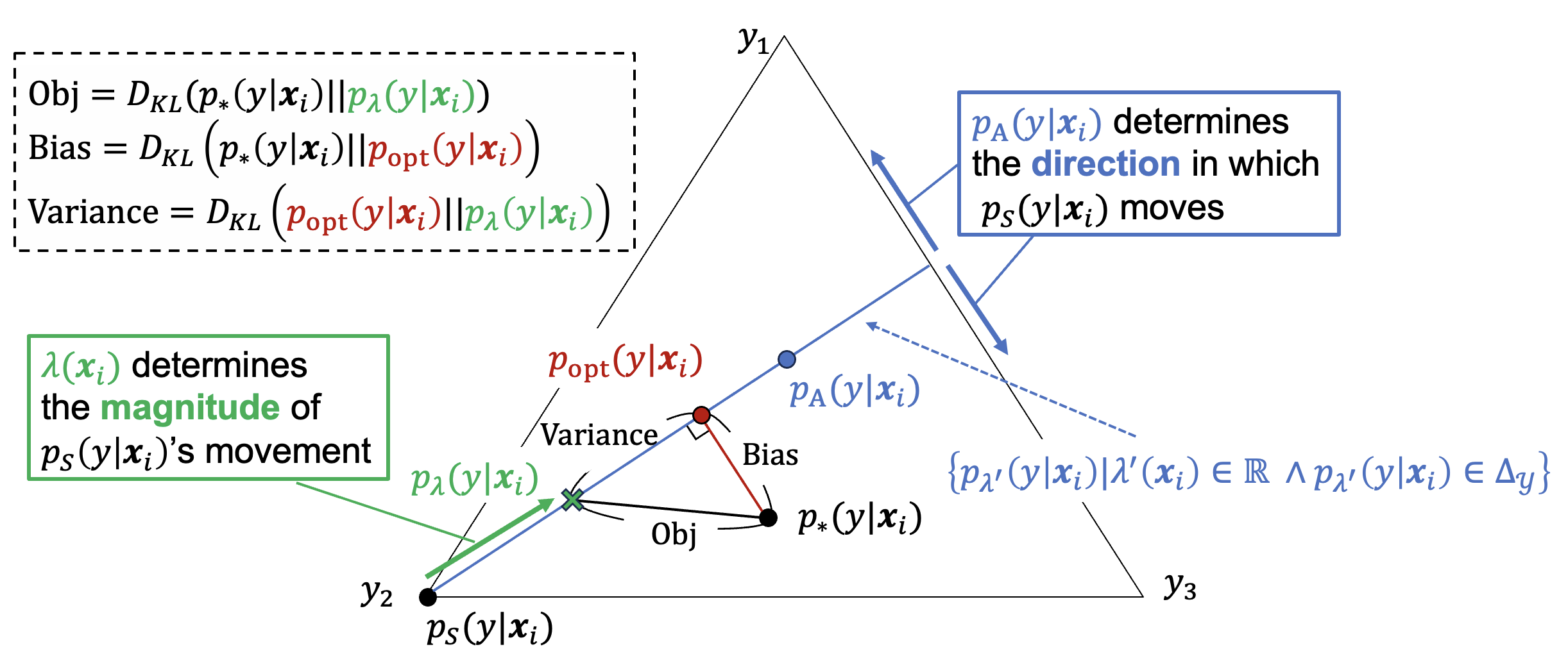}
    \caption{Geometric interpretation for our framework.}
    \label{fig:geometric_interpretation}
\end{figure*}

\subsection{Generalization Error Analysis}
\label{subsec:error_analysis}

In this section, we theoretically analyze how refinements in $p_{\lambda}(y|\bm{x})$, that is, improvements in $p_{\mathrm{A}}(y|\bm{x})$ and $\lambda(\bm{x})$, affect the generalization performance of a classification model trained using $p_{\lambda}(y|\bm{x})$.

We consider training a classification model using $p_{\lambda}(y|\bm{x})$, obtained via the framework in Section~\ref{sec:formulation}, as supervision under the ERM.
First, given a dataset $S = \{(\bm{x}_i, y_i)\}_{i=1}^n$ where each sample is drawn independently from $p_*(\bm{x},y)$, we construct a new dataset $\widehat{S} := \{(\bm{x}_i, p_{\lambda}(y|\bm{x}_i))\}_{i=1}^n$ using the our framework.
We then define the empirical risk on $\widehat{S}$ with respect to a loss function $l$ as follows:
\begin{align}
    \textstyle \widehat{R}_{l, \widehat{S}}(f) := \frac{1}{n}\sum^n_{i=1}\mathbb{E}_{p_{\lambda}(y|\bm{x}_i)}[l(f(\bm{x}_i), Y)],
\end{align}

and obtain the empirical risk minimizer $f_{\widehat{S}} := \arg\min_{f \in \mathcal{F}}\widehat{R}_{l, \widehat{S}}(f)$.
The objective is to make $f_{\widehat{S}}$ approximate the true risk minimizer $f_{\mathcal{F}} := \arg \min_{f \in \mathcal{F}} R_{l}(f)$. 

Our analysis aims to elucidate the relationship between the quality of $p_{\lambda}(y|\bm{x})$ and $R_l(f_{\widehat{S}})$.
Existing analyses of generalization error in training with soft labels yield error bounds that are independent of label quality, and thus cannot capture the relationship we seek to investigate \cite{wang2019theoretical}.
This limitation stems from the fact that error bounds derived from conventional statistical learning theory \cite{mohri2018foundations,shalev2014understanding} inadequately account for label characteristics.
A label-quality-aware analysis requires deriving error bounds that explicitly depend on the quality of the soft labels used for training.
To this end, we establish the following inequality, which relates the difference in risk under different label distributions to the divergence between those distributions.
This result serves as a fundamental tool for deriving error bounds that are sensitive to label quality.
The proof is provided in Appendix~\ref{apdx_subsec:proof_lem:risk_div_inequality}.
\begin{lemma}
    For any probability distribution $p(\bm{x})$ over $\mathcal{X}$, any conditional distributions $p(y|\bm{x})$ and $q(y|\bm{x})$ over $\mathcal{Y}$, any bounded loss function $l \le M_l$, and any measurable $f: \mathcal{X} \rightarrow \mathcal{Y}$, the difference between 
    $R_{l, p}(f):= \mathbb{E}_{p(\bm{x})p(y|\bm{x})}[l(f(\bm{X}),Y)]$ and $R_{l, q}(f):= \mathbb{E}_{p(\bm{x})q(y|\bm{x})}[l(f(\bm{X}),Y)]$ 
    is upper bounded as follows:
    \begin{align}
        \textstyle |\sqrt{R_{l, p}(f)} - \sqrt{R_{l,q}(f)}|
         \le \big\{2 M_l D_{\mathrm{KL}}(p(Y|\bm{X}) || q(Y|\bm{X})) \big\}^{\frac{1}{2}}.
    \label{eq:risk_div_inequality}
    \end{align}
\label{lem:risk_div_inequality}
\end{lemma}
By setting $p(\bm{x}) = p_S(\bm{x})$, $p(y|\bm{x}) = p_{*}(y|\bm{x})$, and $q(y|\bm{x}) = p_{\lambda}(y|\bm{x})$ in Lemma~\ref{lem:risk_div_inequality}, we can relate the difference in empirical risks between using the true soft label $p_{*}(y|\bm{x})$ and the constructed soft label $p_{\lambda}(y|\bm{x})$ to the divergence between the corresponding distributions.
\begin{corollary}
    for any measurable $f \in \mathcal{F}$ and bounded $l \le M_l$, the following holds:
    \begin{align}
    \textstyle \Big|\sqrt{\widehat{R}_{l,S^*}(f)} - \sqrt{\widehat{R}_{l,\widehat{S}}(f)} \Big|
        \le  \big\{2 M_l \frac{1}{n}\sum^n_{i=1}  D_{\mathrm{KL}}(p(Y|\bm{x}_i) || p_{\lambda}(Y|\bm{x}_i)) \big\}^{\frac{1}{2}},    
    \end{align}
    where $S_*:= \{(\bm{x}_i, p_*(y|\bm{x}_i))\}^n_{i=1}$, $\widehat{R}_{l,S^*}(f):=\frac{1}{n}\sum^n_{i=1}\mathbb{E}_{p_*(y|\bm{x}_i)}[l(f(\bm{x}_i),Y)]$.
\label{cor:diff_bound_best_pred}
\end{corollary}
Using this result, we derive the following generalization error bound for $f_{\widehat{S}}$ which is directly influenced by the objective function defined in Eq.~\eqref{eq:mixture_optimization} that quantifies the quality of $p_{\lambda}(y|\bm{x})$.
\begin{theorem}
For any $L_l$-Lipschitz and bounded $l \le M_l$ and any $\delta \in (0,1)$, the following inequality holds with probability at least $1-\delta$:
\begin{align}
\begin{split}
    R_{l}(f_{\widehat{S}}) 
    &\textstyle \le \Big( 4 L_l \mathfrak{R}_{n}(\mathcal{F}) + 2M_l\sqrt{\frac{\log(1/\delta)}{2n}} \Big) \\
    &\textstyle + 2\sqrt{2M_l D(p_{\mathrm{A}}, \lambda)}\Big\{ 2 L_l \mathfrak{R}_{n}(\mathcal{F}) + M_l\sqrt{\frac{\log(1/\delta)}{2n}} + R_l(f_{\mathcal{F}}) +  M_l \sqrt{2 D(p_{\mathrm{A}}, \lambda)} \Big\}^{\frac{1}{2}} \\ 
    &+ M_l \Big( 2D(p_{\mathrm{A}}, \lambda) + \sqrt{2D(p_{\mathrm{A}}, \lambda)} \Big) + R_l(f_{\mathcal{F}}),
\end{split}
\label{eq:error_bound}
\end{align}
where $D(p_{\mathrm{A}}, \lambda):= \frac{1}{n}\sum^n_{i=1} D_{\mathrm{KL}}(p_*(Y|\bm{x}_i) || p_{\lambda}(Y|\bm{x}_i))$.
\label{thm:error_bound}
\end{theorem}
The proof is provided in Appendix~\ref{apdx_subsec:proof_thm:error_bound}.
Here, $\mathfrak{R}_n(\mathcal{F})$ denotes the Rademacher complexity of $\mathcal{F}$ with respect to the true distribution $p_*(\bm{x}, y)$, measuring the capacity of the function class $\mathcal{F}$.
For function classes such as kernel ridge regression or multilayer perceptrons, it is known that $\mathfrak{R}_{n}(\mathcal{F})$ scales as $\mathcal{O}_p(1/n^{1/2})$ and converges to $0$ as $n \rightarrow \infty$ \cite{mohri2018foundations,neyshabur2015norm}.
Therefore, in the following discussion, we assume $\mathfrak{R}_n(\mathcal{F}) = \mathcal{O}_p(1/n^{1/2})$ and $\mathfrak{R}_n(\mathcal{F}) \rightarrow 0$ as $n \rightarrow \infty$.

Theorem~\ref{thm:error_bound} provides a detailed characterization of how improvements in the constructed soft label $p_{\lambda}(y|\bm{x})$, that is, improvements in $p_{\mathrm{A}}(y|\bm{x})$ and $\lambda(\bm{x})$, affect the generalization performance of $f_{\widehat{S}}$.
Specifically, the theorem shows that the mean value of the objective function $D(p_{\mathrm{A}}, \lambda)$, which quantifies the quality of $p_{\lambda}(y|\bm{x})$, directly affects the convergence rate of the error bound.
The first and second terms on the right-hand side of Eq.~\eqref{eq:error_bound} scale as $\mathcal{O}_p(1/n^{1/2})$ and $\mathcal{O}_p(1/n^{1/4})$, respectively.
Notably, $D(p_{\mathrm{A}}, \lambda)$ appears as the coefficient of the second term.
Hence, a larger $D(p_{\mathrm{A}}, \lambda)$ causes the slower-converging second term to dominate the bound.
Moreover, the quality of $p_{\lambda}(y|\bm{x})$ influences not only the convergence rate but also the asymptotic value of the error bound.
As $n \rightarrow \infty$, both $\mathfrak{R}_n(\mathcal{F}) \rightarrow 0$ and $\sqrt{\log(1/\delta)/2n} \rightarrow 0$, and thus the error bound in Eq.~\eqref{eq:error_bound} converges to a quantity that depends solely on $D(p_{\mathrm{A}}, \lambda)$ and $R_l(f_{\mathcal{F}})$.
Therefore, the quality of $p_{\lambda}$ also affects the asymptotic value of the error bound.
Taken together, these results demonstrate that refining $p_{\lambda}(y|\bm{x})$ contributes to improving both the convergence rate and the asymptotic value of the error bound for $f_{\widehat{S}}$.
This result provides theoretical justification that the proposed framework can enhance classification performance even with a limited number of training samples.

\section{Experiments}
\label{sec:experiments}

In this section, we experimentally validate our theoretical insights developed in Section~\ref{subsec:interpret_formulation} by designing $p_{\mathrm{A}}(y|\bm{x}_i)$ and evaluating it.
In Section~\ref{subsec:setting_method}, we introduce specific designs of $p_{\mathrm{A}}(y|\bm{x}_i)$ and a simple method that utilizes them.
Section~\ref{subsec:exp_settings_and_result} outlines the experimental setup and presents the evaluation results.

\subsection{Designing Additional Supervision and Introducing Concise Methods}
\label{subsec:setting_method}

In Section~\ref{subsec:interpret_formulation}, we showed that $p_{\mathrm{A}}(y|\bm{x}_i)$ only needs to convey information about the non-hard-labeled classes.
Based on this, we designed two types of $p_{\mathrm{A}}(y|\bm{x}_i)$ that can be easily provided from external sources and are expected to encode meaningful information about non-hard-labeled classes:

1. {\it Top-1 class among Other Classes (T1OC)}: Among the classes excluding $y_i$, we identify the most probable class $y_i'$ and define $p_{\mathrm{A}}(y|\bm{x}_i)$ such that $p_{\mathrm{A}}(y'_i|\bm{x}_i)=1$.

2. {\it Top-2 class among Other Classes (T2OC)}: Among the classes excluding $y_i$, we select the two most probable labels $y_i'$ and $y_i''$, and define $p_{\mathrm{A}}(y|\bm{x}_i)$ such that $p_{\mathrm{A}}(y'_i|\bm{x}_i)=p_{\mathrm{A}}(y''_i|\bm{x}_i)=0.5$.

These two types of additional supervision correspond to the rightmost two examples in Figure~\ref{fig:eg_additional_supervision}.
Since these forms of $p_{\mathrm{A}}(y|\bm{x}_i)$ do not rely on detailed information such as exact class probabilities or their relative order, they are expected to be applicable in a broad range of scenarios.
Within our framework, the simplest strategy for incorporating $p_{\mathrm{A}}(y|\bm{x}_i)$ is to fix $\lambda(\bm{x}_i)$ to a constant value for all instances, and then construct $p_{\lambda}(y|\bm{x}_i)$ using the affine combination in Eq.~\eqref{eq:linear_comb}.
As suggested by our theoretical analysis, $p_{\mathrm{A}}(y|\bm{x}_i)$ mainly contributes to bias reduction, while $\lambda(\bm{x}_i)$ affects the variance. 
From this perspective, the method can be interpreted as primarily aiming to reduce the bias term.
This implementation constitutes a minimal example of our framework and facilitates a systematic investigation of the effect of $p_{\mathrm{A}}(y|\bm{x}_i)$.

\subsection{Experimental Setting and Experimental Results}
\label{subsec:exp_settings_and_result}

\textbf{Datasets.} 
Four image datasets were used: MNIST~\cite{mnist}, Fashion-MNIST~\cite{fashion_mnist}, Kuzushiji-MNIST~\cite{kuzushiji_mnist}, and CIFAR-10~\cite{cifar10}, along with four datasets from the UCI repository \cite{Dua:2019}: HAR~\cite{har}, Letter~\cite{letter}, Optdigits~\cite{optical}, and Pendigits~\cite{pendigits}.
Details of these datasets are provided in Appendix~\ref{apdx:subsec_dataset}.
To simulate a probabilistic labeling scenario under the assumption that hard labels follow the true label distribution $p_*(y|\bm{x})$, we first trained a neural network as a label generation model and treat it as $p_*(y|\bm{x})$ for each dataset.
Based on the output distribution of the label generation model, each instance was probabilistically assigned a hard label.
The model design and key characteristics of the label generation models are provided in Appendix~\ref{apdx:subsec_label_generation}.

\textbf{Compared Method.} 
As baselines for comparison with T1OC and T2OC, the following methods were employed: {\it Hard}, {\it Soft}, label smoothing (LS) \cite{Szegedy_2016_CVPR}, self-distillation (SD) \cite{furlanello2018born}, and Label Enhancement (LE) \cite{xu2021label,zheng2023label}.
For T1OC and T2OC, the mixing coefficient was fixed as $\lambda(\bm{x}_i)=0.9, \forall \in [n]$ and $p_{\mathrm{A}}(y|\bm{x})$ was assigned based on $p_{*}(y|\bm{x})$.
A detailed comparison of different $\lambda(\bm{x}_i)$ values is presented in Appendix~\ref{apdx:subsec_compare_lambda}.
{\it Hard} uses only the observed hard labels for training.
{\it Soft} uses $p_*(y|\bm{x})$ as a soft label for training.
This evaluation aims to assess (i) the extent to which T1OC and T2OC outperform {\it Hard}, and (ii) how closely their performance approaches that of {\it Soft}.
Additionally, T1OC and T2OC were compared with soft labeling methods that generate soft labels solely from hard labels, without using any additional supervision.
To this end, LS, SD, and LE are used as representative baselines.
For LS, the confidence for the hard label was set to $0.9$ to match the certainty level used in T1OC and T2OC.
For SD, the method of Furlanello et al. (2018) \cite{furlanello2018born} was followed, in which a teacher model is first trained using hard labels and its outputs are used as soft labels for training a student model, where the teacher and student models are implemented with the same architecture, optimization method, and hyperparameters.
Here, the use of the teacher model is regarded as a soft labeling method.
For LE, two approaches were adopted: GLLE \cite{xu2021label}, which reconstructs label distributions using convex optimization, and LIB \cite{zheng2023label}, which uses neural networks for this purpose.
Further implementation details of these methods are provided in Appendix~\ref{apdx:subsec_detail_compared_method}.

\textbf{Learning.} 
Classification models were trained using the soft labels generated by each method. 
For the image datasets, a multi-layer perceptron with four hidden layers of width $5000$ was employed, while for the UCI datasets, a single hidden layer of the same width was used; in both cases, ReLU was adopted as the activation function.
Cross-entropy loss was used, and the models were optimized using the Adam \cite{kingma2014adam} optimizer with a learning rate of $0.0005$, batch size of $32$, $100$ epochs, and weight decay of $0.0002$.
The test dataset were disjoint from the training dataset: $10000$ fixed samples for image datasets, and half of the total samples for UCI datasets.
To simulate scenarios with a small number of training instances, the training data size was varied from $1000$ to $10000$ for image datasets and from $500$ to $2500$ for UCI datasets.
Each experimental setting was repeated five times by varying the random seeds used for training data selection, each comparison method, and model training, and we report the mean and twice the standard deviation.

\textbf{Results: Comparison with All Methods.}
Figure~\ref{fig:compare_size_acc_main} presents the generalization performance of the trained classification models across all eight datasets, evaluated under varying the training data size.
GLLE and LIB were excluded from Figure~\ref{fig:compare_size_acc_main}, as their performance was comparable to or worse than that of LS and exhibited high variance.
Comprehensive results for these methods are reported in Appendix~\ref{apdx:subsec_compare_le}.
As shown in Figure~\ref{fig:compare_size_acc_main}, T1OC and T2OC consistently outperform all baseline methods except {\it Soft}, across all datasets and training data sizes.
These results demonstrate that even the simplest forms of our introduced T1OC and T2OC can substantially improve classification accuracy when guided by our theoretical framework.

\textbf{Results: Comparison between T1OC and T2OC.}
Figure~\ref{fig:compare_size_acc_main} also reveals a notable trend: T2OC generally perform better on image datasets, whereas T1OC tends to outperform or match it on UCI datasets.
This may be attributed to the fact that image datasets often contain multiple plausible alternative classes besides the hard-labeled one.
For example, in the MNIST dataset containing digit images (0 to 9), images labeled as ``1'' may resemble ``7'' or ``9,'' highlighting the benefit of incorporating multiple plausible alternatives.
This observation illustrates that the suitability of $p_{\mathrm{A}}(y|\bm{x})$ depends on the true underlying data distribution, characterized by $p_*(y|\bm{x})$.

\begin{figure*}[t]
    \centering
    \includegraphics[width=0.97\linewidth]{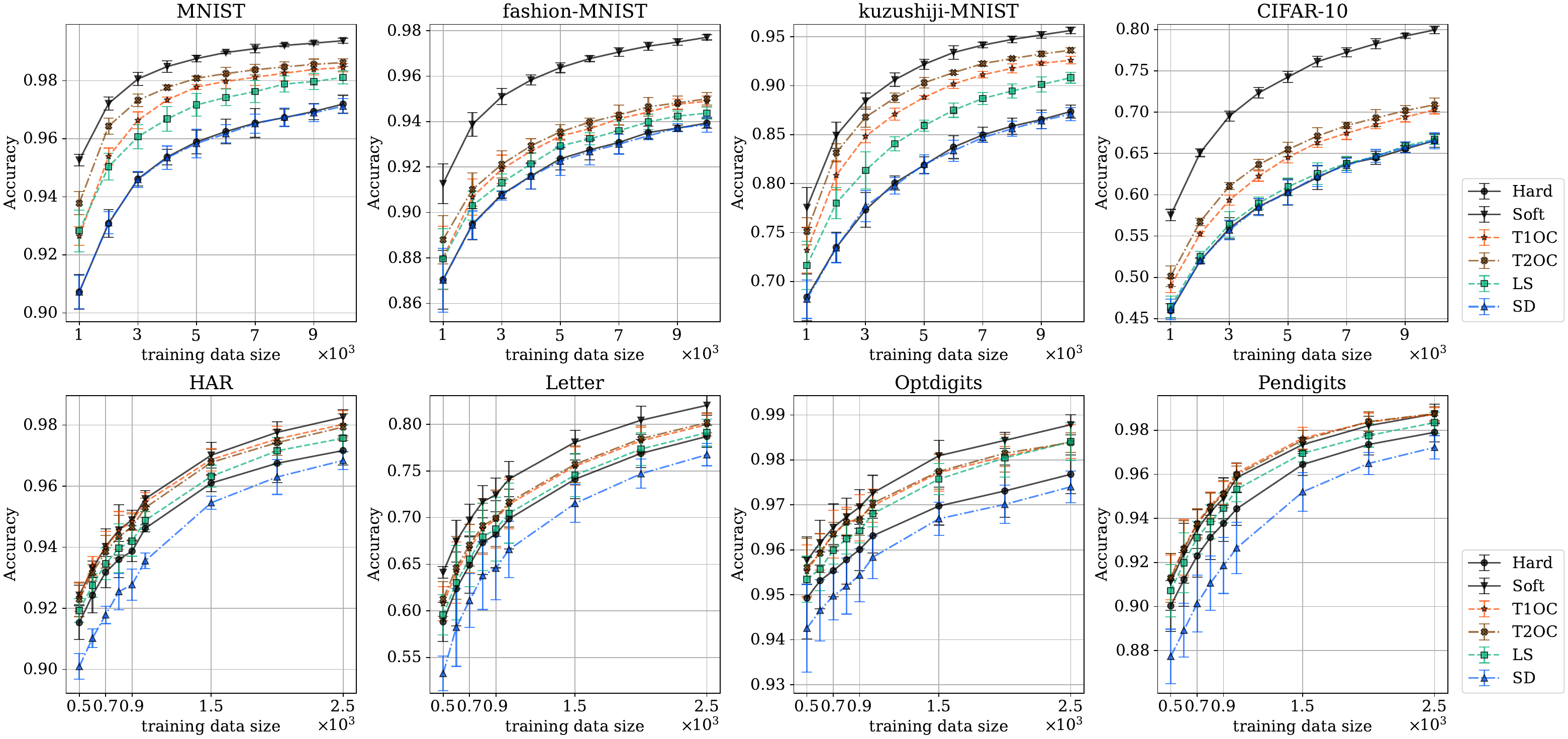}
    \caption{
             Comparison of the classification accuracy of models trained using the labels produced by each method on the image datasets (upper) and UCI datasets (lower).
             The plotted points represent the average of five trials, and the error bars indicate twice the standard deviation.
             }
    \label{fig:compare_size_acc_main}
\end{figure*}

\section{Conclusion}
\label{sec:conclusion}

In this study, we proposed a novel theoretical framework that incorporates additional supervision, represented by $p_{\mathrm{A}}(y|\bm{x})$, in conjunction with hard labels for each instance.
The proposed framework enhances instance-wise supervision by an affine combination $p_{\mathrm{A}}(y|\bm{x})$ and the deterministic distribution representing hard label.
We formulated this framework as an estimation problem of the true label distribution.
We subsequently derived a bias-variance decomposition of the objective function, clarifying the detailed relationships between the bias and variance terms and both $p_{\mathrm{A}}(y|\bm{x})$ and its mixing coefficient $\lambda(\bm{x})$.
Our analysis showed that $p_{\mathrm{A}}(y|\bm{x})$ and $\lambda(\bm{x})$ independently reduce the bias and variance terms, respectively, and can be improved separately.
Furthermore, we demonstrated that $p_{\mathrm{A}}(y|\bm{x})$ requires no information about the hard-labeled class, relying solely on the distribution over the non-hard-labeled classes.
We also showed that optimizing $\lambda(\bm{x})$ is equivalent to estimating the true probability of the hard-labeled class.
In addition, we performed a generalization error analysis to examine the relationship between label quality and classification model's performance, revealing that both $p_{\mathrm{A}}(y|\bm{x})$ and $\lambda(\bm{x})$ influence the convergence rate and asymptotic value of the error bound for the trained classification model.
Finally, based on our theoretical insights, we introduced a simple method that selects one or two most probable non-hard-labeled classes for $p_{\mathrm{A}}(y|\bm{x})$, with a fixed $\lambda(\bm{x})$, and experimentally validated that even such simple methods can improve classification accuracy.
Future work involves developing more practical methods for designing $p_{\mathrm{A}}(y|\bm{x})$ and optimizing $\lambda(\bm{x})$, building upon in our theory.

\section*{Acknowledgment}
This work was supported in part by the Japan Society for the Promotion of Science through Grants-in-Aid for Scientific Research (C) (23K11111).

\bibliographystyle{unsrt}

\bibliography{reference}

\begin{thebibliography}{10}

\bibitem{sheng2008get}
Victor~S Sheng, Foster Provost, and Panagiotis~G Ipeirotis.
\newblock Get another label? improving data quality and data mining using multiple, noisy labelers.
\newblock In {\em Proceedings of the 14th ACM SIGKDD international conference on Knowledge discovery and data mining}, pages 614--622, 2008.

\bibitem{rodrigues2018deep}
Filipe Rodrigues and Francisco Pereira.
\newblock Deep learning from crowds.
\newblock In {\em Proceedings of the AAAI conference on artificial intelligence}, volume~32, 2018.

\bibitem{peterson_2019_ICCV}
Joshua~C. Peterson, Ruairidh~M. Battleday, Thomas~L. Griffiths, and Olga Russakovsky.
\newblock Human uncertainty makes classification more robust.
\newblock In {\em Proceedings of the IEEE/CVF International Conference on Computer Vision (ICCV)}, October 2019.

\bibitem{battleday2020capturing}
Ruairidh~M Battleday, Joshua~C Peterson, and Thomas~L Griffiths.
\newblock Capturing human categorization of natural images by combining deep networks and cognitive models.
\newblock {\em Nature communications}, 11(1):5418, 2020.

\bibitem{uma2020case}
Alexandra Uma, Tommaso Fornaciari, Dirk Hovy, Silviu Paun, Barbara Plank, and Massimo Poesio.
\newblock A case for soft loss functions.
\newblock In {\em Proceedings of the AAAI Conference on Human Computation and Crowdsourcing}, volume~8, pages 173--177, 2020.

\bibitem{fornaciari2021beyond}
Tommaso Fornaciari, Alexandra Uma, Silviu Paun, Barbara Plank, Dirk Hovy, Massimo Poesio, et~al.
\newblock Beyond black \& white: Leveraging annotator disagreement via soft-label multi-task learning.
\newblock In {\em Proceedings of the 2021 Conference of the North American Chapter of the Association for Computational Linguistics: Human Language Technologies}. Association for Computational Linguistics, 2021.

\bibitem{davani2022dealing}
Aida~Mostafazadeh Davani, Mark D{\'\i}az, and Vinodkumar Prabhakaran.
\newblock Dealing with disagreements: Looking beyond the majority vote in subjective annotations.
\newblock {\em Transactions of the Association for Computational Linguistics}, 10:92--110, 2022.

\bibitem{collins2022eliciting}
Katherine~M Collins, Umang Bhatt, and Adrian Weller.
\newblock Eliciting and learning with soft labels from every annotator.
\newblock In {\em Proceedings of the AAAI conference on human computation and crowdsourcing}, volume~10, pages 40--52, 2022.

\bibitem{wu2023don}
Ben Wu, Yue Li, Yida Mu, Carolina Scarton, Kalina Bontcheva, and Xingyi Song.
\newblock Don’t waste a single annotation: improving single-label classifiers through soft labels.
\newblock In {\em Findings of the Association for Computational Linguistics: EMNLP 2023}, pages 5347--5355, 2023.

\bibitem{cour2011learning}
Timothee Cour, Ben Sapp, and Ben Taskar.
\newblock Learning from partial labels.
\newblock {\em J. Mach. Learn. Res.}, 12(null):1501^^e2^^80^^931536, July 2011.

\bibitem{ishida2017cll}
Takashi Ishida, Gang Niu, Weihua Hu, and Masashi Sugiyama.
\newblock Learning from complementary labels.
\newblock In I.~Guyon, U.~Von Luxburg, S.~Bengio, H.~Wallach, R.~Fergus, S.~Vishwanathan, and R.~Garnett, editors, {\em Advances in Neural Information Processing Systems}, volume~30. Curran Associates, Inc., 2017.

\bibitem{tang2025confidence}
Xijia Tang, Chao Xu, Hong Tao, Xiaoyu Ma, and Chenping Hou.
\newblock Confidence-based pu learning with instance-dependent label noise.
\newblock {\em IEEE Transactions on Neural Networks and Learning Systems}, pages 1--15, 2025.

\bibitem{cao2021sconf}
Yuzhou Cao, Lei Feng, Yitian Xu, Bo~An, Gang Niu, and Masashi Sugiyama.
\newblock Learning from similarity-confidence data.
\newblock In Marina Meila and Tong Zhang, editors, {\em Proceedings of the 38th International Conference on Machine Learning}, volume 139 of {\em Proceedings of Machine Learning Research}, pages 1272--1282. PMLR, 18--24 Jul 2021.

\bibitem{wang2023confdiff}
Wei Wang, Lei Feng, Yuchen Jiang, Gang Niu, Min-Ling Zhang, and Masashi Sugiyama.
\newblock Binary classification with confidence difference.
\newblock In A.~Oh, T.~Naumann, A.~Globerson, K.~Saenko, M.~Hardt, and S.~Levine, editors, {\em Advances in Neural Information Processing Systems}, volume~36, pages 5936--5960. Curran Associates, Inc., 2023.

\bibitem{wang2019theoretical}
Jing Wang and Xin Geng.
\newblock Theoretical analysis of label distribution learning.
\newblock In {\em Proceedings of the AAAI Conference on Artificial Intelligence}, volume~33, pages 5256--5263, 2019.

\bibitem{cifar10}
Alex Krizhevsky, Geoffrey Hinton, et~al.
\newblock Learning multiple layers of features from tiny images.
\newblock 2009.

\bibitem{mohri2018foundations}
Mehryar Mohri.
\newblock Foundations of machine learning, 2018.

\bibitem{shalev2014understanding}
Shai Shalev-Shwartz and Shai Ben-David.
\newblock {\em Understanding machine learning: From theory to algorithms}.
\newblock Cambridge university press, 2014.

\bibitem{xu2020variational}
Ning Xu, Jun Shu, Yun-Peng Liu, and Xin Geng.
\newblock Variational label enhancement.
\newblock In Hal~Daum^^c3^^a9 III and Aarti Singh, editors, {\em Proceedings of the 37th International Conference on Machine Learning}, volume 119 of {\em Proceedings of Machine Learning Research}, pages 10597--10606. PMLR, 13--18 Jul 2020.

\bibitem{xu2021label}
Ning Xu, Yun-Peng Liu, and Xin Geng.
\newblock Label enhancement for label distribution learning.
\newblock {\em IEEE Transactions on Knowledge and Data Engineering}, 33(4):1632--1643, 2021.

\bibitem{qianghai2023generalized}
Qinghai Zheng, Jihua Zhu, Haoyu Tang, Xinyuan Liu, Zhongyu Li, and Huimin Lu.
\newblock Generalized label enhancement with sample correlations.
\newblock {\em IEEE Transactions on Knowledge and Data Engineering}, 35(1):482--495, 2023.

\bibitem{zheng2023label}
Qinghai Zheng, Jihua Zhu, and Haoyu Tang.
\newblock Label information bottleneck for label enhancement.
\newblock In {\em Proceedings of the IEEE/CVF Conference on Computer Vision and Pattern Recognition}, pages 7497--7506, 2023.

\bibitem{hinton2015distilling}
Geoffrey Hinton, Oriol Vinyals, and Jeff Dean.
\newblock Distilling the knowledge in a neural network.
\newblock {\em arXiv preprint arXiv:1503.02531}, 2015.

\bibitem{gou2021knowledge}
Jianping Gou, Baosheng Yu, Stephen~J Maybank, and Dacheng Tao.
\newblock Knowledge distillation: A survey.
\newblock {\em International Journal of Computer Vision}, 129(6):1789--1819, 2021.

\bibitem{furlanello2018born}
Tommaso Furlanello, Zachary Lipton, Michael Tschannen, Laurent Itti, and Anima Anandkumar.
\newblock Born again neural networks.
\newblock In Jennifer Dy and Andreas Krause, editors, {\em Proceedings of the 35th International Conference on Machine Learning}, volume~80 of {\em Proceedings of Machine Learning Research}, pages 1607--1616. PMLR, 10--15 Jul 2018.

\bibitem{Szegedy_2016_CVPR}
Christian Szegedy, Vincent Vanhoucke, Sergey Ioffe, Jon Shlens, and Zbigniew Wojna.
\newblock Rethinking the inception architecture for computer vision.
\newblock In {\em Proceedings of the IEEE Conference on Computer Vision and Pattern Recognition (CVPR)}, June 2016.

\bibitem{xu2019partial}
Ning Xu, Jiaqi Lv, and Xin Geng.
\newblock Partial label learning via label enhancement.
\newblock In {\em Proceedings of the AAAI Conference on artificial intelligence}, volume~33, pages 5557--5564, 2019.

\bibitem{muller2019does}
Rafael M{\"u}ller, Simon Kornblith, and Geoffrey~E Hinton.
\newblock When does label smoothing help?
\newblock {\em Advances in neural information processing systems}, 32, 2019.

\bibitem{neyshabur2015norm}
Behnam Neyshabur, Ryota Tomioka, and Nathan Srebro.
\newblock Norm-based capacity control in neural networks.
\newblock In Peter Gr^^c3^^bcnwald, Elad Hazan, and Satyen Kale, editors, {\em Proceedings of The 28th Conference on Learning Theory}, volume~40 of {\em Proceedings of Machine Learning Research}, pages 1376--1401, Paris, France, 03--06 Jul 2015. PMLR.

\bibitem{mnist}
Yann LeCun, L{\'e}on Bottou, Yoshua Bengio, and Patrick Haffner.
\newblock Gradient-based learning applied to document recognition.
\newblock {\em Proceedings of the IEEE}, 86(11):2278--2324, 1998.

\bibitem{fashion_mnist}
Han Xiao, Kashif Rasul, and Roland Vollgraf.
\newblock Fashion-mnist: a novel image dataset for benchmarking machine learning algorithms.
\newblock {\em arXiv preprint arXiv:1708.07747}, 2017.

\bibitem{kuzushiji_mnist}
Tarin Clanuwat, Mikel Bober-Irizar, Asanobu Kitamoto, Alex Lamb, Kazuaki Yamamoto, and David Ha.
\newblock Deep learning for classical japanese literature.
\newblock {\em arXiv preprint arXiv:1812.01718}, 2018.

\bibitem{Dua:2019}
Dheeru Dua and Casey Graff.
\newblock {UCI} machine learning repository, 2017.

\bibitem{har}
Anguita Davide Ghio Alessandro Oneto~Luca Reyes-Ortiz, Jorge and Xavier Parra.
\newblock {Human Activity Recognition Using Smartphones}.
\newblock UCI Machine Learning Repository, 2013.
\newblock {DOI}: https://doi.org/10.24432/C54S4K.

\bibitem{letter}
David Slate.
\newblock {Letter Recognition}.
\newblock UCI Machine Learning Repository, 1991.
\newblock {DOI}: https://doi.org/10.24432/C5ZP40.

\bibitem{optical}
E.~Alpaydin and C.~Kaynak.
\newblock {Optical Recognition of Handwritten Digits}.
\newblock UCI Machine Learning Repository, 1998.
\newblock {DOI}: https://doi.org/10.24432/C50P49.

\bibitem{pendigits}
E.~Alpaydin and Fevzi. Alimoglu.
\newblock {Pen-Based Recognition of Handwritten Digits}.
\newblock UCI Machine Learning Repository, 1996.
\newblock {DOI}: https://doi.org/10.24432/C5MG6K.

\bibitem{kingma2014adam}
Diederik Kingma and Jimmy Ba.
\newblock Adam: A method for stochastic optimization.
\newblock {\em International Conference on Learning Representations}, 12 2014.

\bibitem{amari2000methods}
Shun-ichi Amari and Hiroshi Nagaoka.
\newblock {\em Methods of information geometry}, volume 191.
\newblock American Mathematical Soc., 2000.

\end{thebibliography}

\newpage
\appendix

\section{Proof of Section~\ref{subsec:interpret_formulation}}
\label{apdx_sec:proof_interpretation}

\subsection{Proof of Theorem~\ref{thm:bias_var_decomposition}}
\label{apdx_subsec:proof_thm:bias_var_decomposition}

\begin{proof}[Proof of Theorem~\ref{thm:bias_var_decomposition}]

Figure~\ref{fig:geometric_proof} provides a supplemental illustration for the proof.
Since $p_{\lambda}(y|\bm{x}_i)$ is an affine combination of $p_{S}(y|\bm{x}_i)$ and $p_{\mathrm{A}}(y|\bm{x}_i)$ and satisfies $p_{\lambda}(y|\bm{x}_i) \in \Delta_{\mathcal{Y}}$, the set of possible values that $p_{\lambda}(y|\bm{x}_i)$ can take (depicted by the blue line) lies on an $m$-geodesic.
By the definition of $p_{\mathrm{opt}}(y|\bm{x}_i)$, the red line connecting $p_{*}(y|\bm{x}_i)$ and $p_{\mathrm{opt}}(y|\bm{x}_i)$ is an $e$-geodesic.
Therefore, by the generalized Pythagorean theorem \cite{amari2000methods}, the decomposition in Eq.~\eqref{eq:bias_var_decomposition} holds among $p_{*}(y|\bm{x}_i)$, $p_{\mathrm{opt}}(y|\bm{x}_i)$, and $p_{\lambda}(y|\bm{x}_i)$.
\end{proof}

\begin{figure*}[h]
    \centering
    \includegraphics[width=\linewidth]{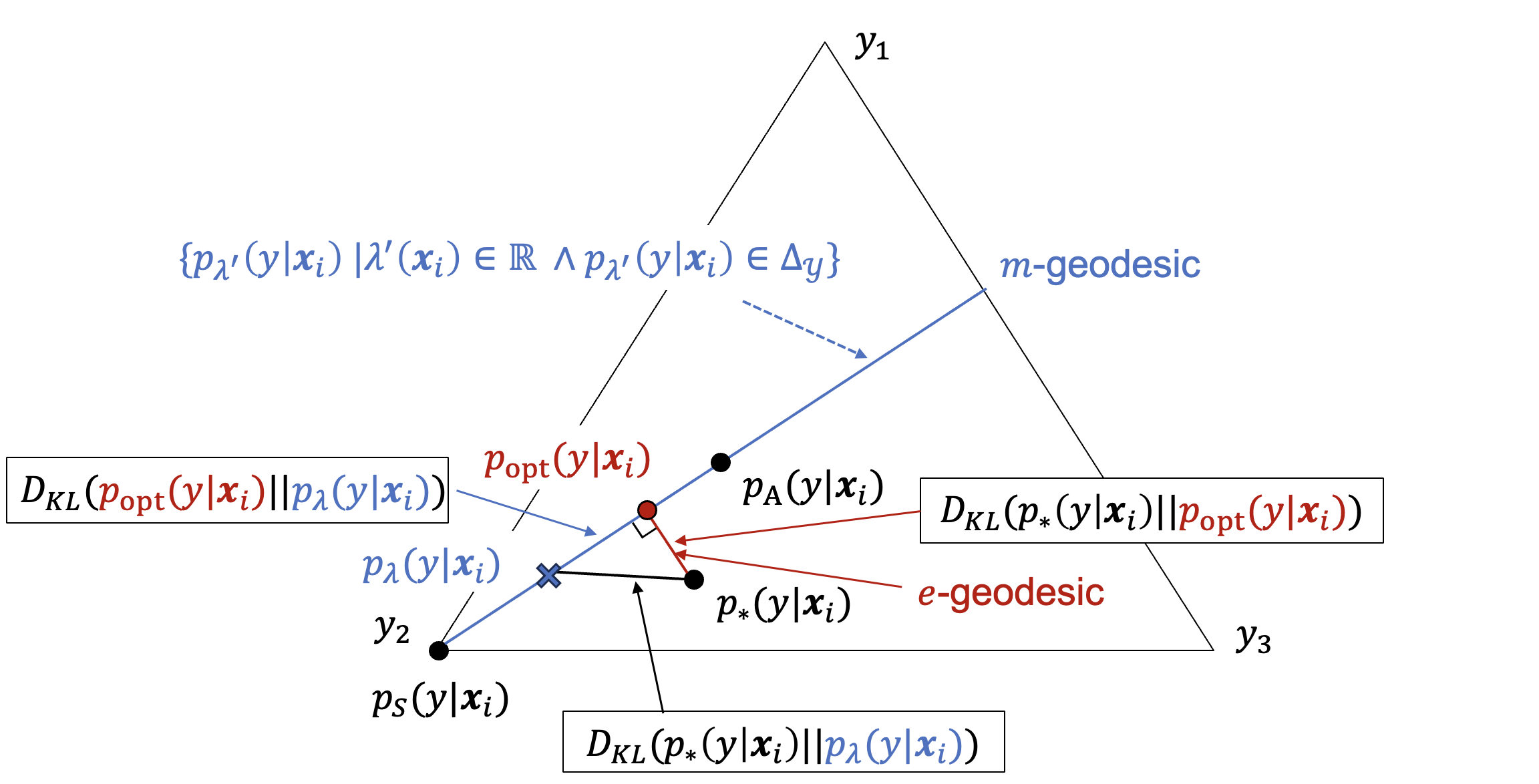}
    \caption{Geometric interpretation for our framework.}
    \label{fig:geometric_proof}
\end{figure*}

\subsection{Proof of Lemma~\ref{lem:dist_equations}}
\label{apdx_subsec:proof_lem:dist_equations}

\begin{proof}[Proof of Lemma~\ref{lem:dist_equations}]

We begin by proving Eq.~\eqref{eq:equal_lam_add}.
Since $p_{\lambda}(y|\bm{x}_i)$ is an affine combination of the deterministic distribution $p_{S}(y|\bm{x}_i)$, which satisfies $p_{S}(y_i|\bm{x}_i) = 1$, and $p_{\mathrm{A}}(y|\bm{x}_i)$, it follows that for any $i \in [n]$ and any $\lambda$ such that $p_{\lambda}(y|\bm{x}_i) \in \Delta_{\mathcal{Y}}$, we have:
\begin{align}
    p_{\lambda}(y|\bm{x}_i) = \begin{cases}
        \lambda(\bm{x}_i) + (1-\lambda(\bm{x}_i))p_{\mathrm{A}}(y|\bm{x}_i) & \text{if } y=y_i \\
        (1-\lambda(\bm{x}_i))p_{\mathrm{A}}(y|\bm{x}_i) & \text{if } y \neq y_i
    \end{cases}.
\label{eq:lam_detail}
\end{align}
Thus, for any $\lambda$ such that $p_{\lambda}(y|\bm{x}_i) \in \Delta_{\mathcal{Y}}$, the following holds:
\begin{align}
    & \frac{1}{p_{\mathrm{A}}(Y\neq y_i | \bm{x}_i)} p_{\mathrm{A}}(y'| \bm{x}_i) = \frac{1}{p_{\lambda}(Y\neq y_i | \bm{x}_i)} p_{\lambda}(y'| \bm{x}_i), ~~~~\forall y' \in \mathcal{Y}\setminus \{y_i\}, 
\end{align}
which proves Eq.~\eqref{eq:equal_lam_add}.

Moreover, since $p_{\mathrm{opt}}(y|\bm{x}_i)=p_{\lambda=\lambda^*}(y|\bm{x})$, it follows similarly that for any $y' \in \mathcal{Y} \setminus \{ y_i \}$, we have:
\begin{align}
\begin{split}
    \frac{1}{p_{\mathrm{A}}(Y\neq y_i | \bm{x}_i)} p_{\mathrm{A}}(y'| \bm{x}_i) 
    &= \frac{1}{p_{\mathrm{opt}}(Y\neq y_i | \bm{x}_i)} p_{\mathrm{opt}}(y'| \bm{x}_i) \\
    &= \frac{1}{p_{\lambda}(Y\neq y_i | \bm{x}_i)} p_{\lambda}(y'| \bm{x}_i),
\end{split}
\label{eq:equal_lam_opt_add}
\end{align}
where $p_{\mathrm{opt}}(Y\neq y_i|\bm{x}_i) := \sum_{y \in \mathcal{Y}\setminus\{y_i\}}p_{\mathrm{opt}}(y|\bm{x}_i)$.

Next, we prove Eq.~\eqref{eq:equal_star_opt}.
For any $i \in [n]$, the objective function in Eq.~\eqref{eq:mixture_optimization} can be expanded as:
\begin{align}
    &D_{\mathrm{KL}}(p_*(Y|\bm{x}_i) || p_{\lambda}(Y|\bm{x}_i))   \notag \\
    &= p_*(y_i|\bm{x}_i) \log \frac{p_*(y_i|\bm{x}_i)}{p_{\lambda}(y_i|\bm{x}_i)} 
    + \sum_{y \in \mathcal{Y}\setminus \{y_i\}} p_*(y|\bm{x}_i) \log \frac{p_*(y|\bm{x}_i)}{p_{\lambda}(y|\bm{x}_i)} \notag \\
    &=p_*(y_i|\bm{x}_i) \log \frac{p_*(y_i|\bm{x}_i)}{p_{\lambda}(y_i|\bm{x}_i)} \notag \\
    &~~~~+ \sum_{y \in \mathcal{Y}\setminus \{y_i\}} p_*(y|\bm{x}_i) 
                \Bigg\{
                    \log \frac{p_*(y|\bm{x}_i)/p_*(Y\neq y_i|\bm{x}_i)}{p_{\lambda}(y|\bm{x}_i)/p_{\lambda}(Y\neq y_i|\bm{x}_i)}
                    + \log \frac{p_*(Y\neq y_i|\bm{x}_i)}{p_{\lambda}(Y\neq y_i|\bm{x}_i)}
                \Bigg\}
                \notag \\
    &=p_*(y_i|\bm{x}_i) \log \frac{p_*(y_i|\bm{x}_i)}{p_{\lambda}(y_i|\bm{x}_i)} 
        + p_*(Y\neq y_i|\bm{x}_i)\log \frac{p_*(Y\neq y_i|\bm{x}_i)}{p_{\lambda}(Y\neq y_i|\bm{x}_i)} \notag \\
    &~~~~~+ \sum_{y \in \mathcal{Y}\setminus \{y_i\}} p_*(y|\bm{x}_i) \log \frac{p_*(y|\bm{x}_i)/p_*(Y\neq y_i|\bm{x}_i)}{p_{\lambda}(y|\bm{x}_i)/p_{\lambda}(Y\neq y_i|\bm{x}_i)}, \notag
\end{align}
where $p_{*}(Y\neq y_i|\bm{x}_i) := \sum_{y \in \mathcal{Y}\setminus\{y_i\}}p_{*}(y|\bm{x}_i)$.

Hence, we obtain:
\begin{align}
\begin{split}
    &D_{\mathrm{KL}}(p_*(Y|\bm{x}_i) || p_{\lambda}(Y|\bm{x}_i))   \\
    &=p_*(y_i|\bm{x}_i) \log \frac{p_*(y_i|\bm{x}_i)}{p_{\lambda}(y_i|\bm{x}_i)} 
        + (1- p_*(y_i|\bm{x}_i))\log \frac{1-p_*(y_i|\bm{x}_i)}{1-p_{\lambda}(y_i|\bm{x}_i)}  \\
    &~~~~~+ \sum_{y \in\mathcal{Y}\setminus \{y_i\}} p_*(y|\bm{x}_i) \log \frac{p_*(y|\bm{x}_i)/p_*(Y\neq y_i|\bm{x}_i)}{p_{\lambda}(y|\bm{x}_i)/p_{\lambda}(Y\neq y_i|\bm{x}_i)}. \\
\end{split}
\label{eq:obj_representation_lambda}
\end{align}
From Eq.~\eqref{eq:equal_lam_add}, the ratio $p_{\lambda}(y|\bm{x}_i)/p_{\lambda}(Y \neq y_i|\bm{x}_i)$ remains constant for any $\lambda$ such that $p_{\lambda}(y|\bm{x}_i) \in \Delta_{\mathcal{Y}}$, and therefore the third term on the right-hand side of Eq.~\eqref{eq:obj_representation_lambda} is constant with respect to $\lambda$.
Furthermore, the first and second terms on the right-hand side of Eq.~\eqref{eq:obj_representation_lambda} correspond to the conditional KL divergence with respect to $Z_i = \mathbbm{1}_{[Y = y_i]}$, and are thus non-negative.
It follows that the value of $\lambda$ minimizing $D_{\mathrm{KL}}(p_{*}(Y|\bm{x}_i) || p_{\lambda}(Y|\bm{x}_i))$ is achieved when $p_{*}(y_i|\bm{x}_i) = p_{\lambda}(y_i|\bm{x}_i)$.
Since the domain of $\lambda(\bm{x}_i)$ is $\mathbb{R}_+$ for any $i \in [n]$, there exists a $\lambda$ such that $p_*(y_i|\bm{x}_i)=p_{\lambda}(y_i|\bm{x}_i)$ holds for any $i \in [n]$.
Therefore, the optimal $\lambda^* := \argmin_{\lambda}D_{\mathrm{KL}}(p_*(Y|\bm{x}_i)|| p_{\lambda}(y|\bm{x}_i))$ satisfies $p_*(y_i|\bm{x}_i)=p_{\lambda^*}(y_i|\bm{x}_i) ~ (= p_{\mathrm{opt}}(y_i|\bm{x}_i))$.
This establishes Eq.~\eqref{eq:equal_star_opt}:
\begin{align}
    p_*(y_i|\bm{x}_i)=p_{\mathrm{opt}}(y_i|\bm{x}_i), ~~~ \forall i \in [n]. \notag
\end{align}
This also implies that the following equation holds for any $i \in [n]$:
\begin{align}
    p_*(Y\neq y_i | \bm{x}_i) 
    = \sum_{y \in \mathcal{Y}\setminus \{y_i\}} p_*(y|\bm{x}_i) = 1 - p_*(y_i|\bm{x}_i) = 1- p_{\mathrm{opt}}(y_i|\bm{x}_i) = p_{\mathrm{opt}}(Y\neq y_i | \bm{x}_i).
\label{eq:equal_star_opt_neq}
\end{align}
where $p_{\mathrm{opt}}(Y\neq y_i|\bm{x}_i) := \sum_{y \in \mathcal{Y}\setminus\{y_i\}}p_{\mathrm{opt}}(y|\bm{x}_i)$.

\end{proof}

\subsection{Proof of Theorem~\ref{thm:transform_bias_variance}}
\label{apdx_subsec:proof_thm:transform_bias_variance}

\begin{proof}[Proof of Theorem~\ref{thm:transform_bias_variance}]

We now analyze the bias term in Theorem~\ref{thm:bias_var_decomposition}. For any $i \in [n]$, it can be rewritten as:
\begin{align}
    \text{Bias}
    &= \sum_{y \in \mathcal{Y}} p_*(y|\bm{x}_i) \log \frac{p_*(y|\bm{x}_i)}{p_{\mathrm{opt}}(y|\bm{x}_i)}  \notag \\
    &= \sum_{y \in \mathcal{Y}\setminus \{y_i\}} p_*(y|\bm{x}_i) \log \frac{p_*(y|\bm{x}_i)}{p_{\mathrm{opt}}(y|\bm{x}_i)}  \notag \\
    &= p_*(Y\neq y_i | \bm{x}_i) \sum_{y \in \mathcal{Y}\setminus \{y_i\}} \frac{p_*(y|\bm{x}_i)}{p_*(Y\neq y_i | \bm{x}_i)} \log \frac{p_*(y|\bm{x}_i)/p_*(Y\neq y_i|\bm{x}_i)}{p_{\mathrm{opt}}(y|\bm{x}_i)/p_{*}(Y\neq y_i|\bm{x}_i)}  \notag \\
    &= p_*(Y\neq y_i | \bm{x}_i) \sum_{y \in \mathcal{Y}\setminus \{y_i\}} \frac{p_*(y|\bm{x}_i)}{p_*(Y\neq y_i | \bm{x}_i)} \log \frac{p_*(y|\bm{x}_i)/p_*(Y\neq y_i|\bm{x}_i)}{p_{\mathrm{opt}}(y|\bm{x}_i)/p_{\mathrm{opt}}(Y\neq y_i|\bm{x}_i)} \notag \\
    &= p_*(Y\neq y_i | \bm{x}_i) \sum_{y \in \mathcal{Y}\setminus \{y_i\}} \frac{p_*(y|\bm{x}_i)}{p_*(Y\neq y_i | \bm{x}_i)} \log \frac{p_*(y|\bm{x}_i)/p_*(Y\neq y_i|\bm{x}_i)}{p_{\mathrm{A}}(y|\bm{x}_i)/p_{\mathrm{A}}(Y\neq y_i|\bm{x}_i)} \notag \\
    &=p_*(Y\neq y_i | \bm{x}_i) D_{\mathrm{KL}}(p_{*, \neq y_i}(Y|\bm{x}_i) || p_{\mathrm{A}, \neq y_i}(Y|\bm{x}_i)) \notag \\
    &=(1-p_*(y_i | \bm{x}_i)) D_{\mathrm{KL}}(p_{*, \neq y_i}(Y|\bm{x}_i) || p_{\mathrm{A}, \neq y_i}(Y|\bm{x}_i)).
\label{eq:bias_tmp1}
\end{align}
Here, the second equality uses Eq.~\eqref{eq:equal_star_opt} in Lemma~\ref{lem:dist_equations}, which establishes $p_{*}(y_i|\bm{x}_i) = p_{\mathrm{opt}}(y_i|\bm{x}_i)$.
The fourth equality uses Eq.~\eqref{eq:equal_star_opt_neq} in the proof of Lemma~\ref{lem:dist_equations}, which guarantees $p_{*}(Y \neq y_i|\bm{x}_i) = p_{\mathrm{opt}}(Y \neq y_i|\bm{x}_i)$, and the fifth equality uses Eq.~\eqref{eq:equal_lam_add} from the same lemma.
Hence, Eq.~\eqref{eq:bias_representation} is established.

We now consider the variance term in Theorem~\ref{thm:bias_var_decomposition}. For any $i \in [n]$, it can be rewritten as:
\begin{align}
    \text{Variance}
    &= \sum_{y \in \mathcal{Y}} p_{\mathrm{opt}}(y|\bm{x}_i) \log \frac{p_{\mathrm{opt}}(y|\bm{x}_i)}{p_{\lambda}(y|\bm{x}_i)} \notag \\
    &= p_{\mathrm{opt}}(y_i|\bm{x}_i) \log \frac{p_{\mathrm{opt}}(y_i|\bm{x}_i)}{p_{\lambda}(y_i|\bm{x}_i)}  
        + \sum_{y \in \mathcal{Y} \setminus \{y_i\}} p_{\mathrm{opt}}(y|\bm{x}_i) \log \frac{p_{\mathrm{opt}}(y|\bm{x}_i)}{p_{\lambda}(y|\bm{x}_i)}   \notag \\
    &= p_{\mathrm{opt}}(y_i|\bm{x}_i) \log \frac{p_{\mathrm{opt}}(y_i|\bm{x}_i)}{p_{\lambda}(y_i|\bm{x}_i)}  \notag \\
    &~~~ + \sum_{y \in \mathcal{Y} \setminus \{y_i\}} p_{\mathrm{opt}}(y|\bm{x}_i) \Bigg\{ \log \frac{p_{\mathrm{opt}}(y|\bm{x}_i) / p_{\mathrm{opt}}(Y\neq y_i|\bm{x}_i)}{p_{\lambda}(y|\bm{x}_i) / p_{\lambda}(Y\neq y_i|\bm{x}_i)}   
        + \log \frac{p_{\mathrm{opt}}(Y\neq y_i|\bm{x}_i)}{p_{\lambda}(Y\neq y_i|\bm{x}_i)}
        \Bigg\}  \notag  \\
    &= p_{\mathrm{opt}}(y_i|\bm{x}_i) \log \frac{p_{\mathrm{opt}}(y_i|\bm{x}_i)}{p_{\lambda}(y_i|\bm{x}_i)}  
        + \sum_{y \in \mathcal{Y} \setminus \{y_i\}} p_{\mathrm{opt}}(y|\bm{x}_i) \log \frac{p_{\mathrm{opt}}(Y\neq y_i|\bm{x}_i)}{p_{\lambda}(Y\neq y_i|\bm{x}_i)}
         \notag  \\
    &= p_{\mathrm{opt}}(y_i|\bm{x}_i) \log \frac{p_{\mathrm{opt}}(y_i|\bm{x}_i)}{p_{\lambda}(y_i|\bm{x}_i)}  
        + p_{\mathrm{opt}}(Y\neq y_i|\bm{x}_i) \log \frac{p_{\mathrm{opt}}(Y\neq y_i|\bm{x}_i)}{p_{\lambda}(Y\neq y_i|\bm{x}_i)} \notag  \\
    &= p_{\mathrm{opt}}(y_i|\bm{x}_i) \log \frac{p_{\mathrm{opt}}(y_i|\bm{x}_i)}{p_{\lambda}(y_i|\bm{x}_i)}  
        + (1- p_{\mathrm{opt}}(y_i|\bm{x}_i)) \log \frac{1 - p_{\mathrm{opt}}(y_i|\bm{x}_i)}{1- p_{\lambda}(y_i|\bm{x}_i)}. \notag  \\
    &= p_{*}(y_i|\bm{x}_i) \log \frac{p_{*}(y_i|\bm{x}_i)}{p_{\lambda}(y_i|\bm{x}_i)}  
        + (1- p_{*}(y_i|\bm{x}_i)) \log \frac{1 - p_{*}(y_i|\bm{x}_i)}{1- p_{\lambda}(y_i|\bm{x}_i)}. \notag  \\
\end{align}
Here, the fourth equality uses Eq.~\eqref{eq:equal_lam_add} in Lemma~\ref{lem:dist_equations}, and the final inequality uses Eq.~\eqref{eq:equal_star_opt}, which ensures $p_*(y_i|\bm{x}_i) = p_{\mathrm{opt}}(y_i|\bm{x}_i)$.
Thus, Eq.~\eqref{eq:variance_representation} is established.
    
\end{proof}

\section{Proof of Section~\ref{subsec:error_analysis}}
\label{apdx_sec:proof_error_analysis}

\subsection{Proof of Lemma~\ref{lem:risk_div_inequality}}
\label{apdx_subsec:proof_lem:risk_div_inequality}

\begin{proof}[Proof of Lemma~\ref{lem:risk_div_inequality}]

The left-hand side of Eq.~\eqref{eq:risk_div_inequality} can be upper-bounded as follows:
\begin{align}
    &|R_{l,p}(f)  - R_{l,q}(f)| \notag \\
    &= | \mathbb{E}_{p(\bm{x})p(y|\bm{x})}[l(f(\bm{X}), Y)]  - \mathbb{E}_{p(\bm{x})q(y|\bm{x})}[l(f(\bm{X}), Y)] | \notag \\
    &\le \mathbb{E}_{p(\bm{x})}\bigg[\sum_{y \in \mathcal{Y}} l(f(\bm{X}), y) |p(y|\bm{X}) - q(y|\bm{X}) | \bigg] \notag \\
    &=\mathbb{E}_{p(\bm{x})}\bigg[\sum_{y \in \mathcal{Y}} l(f(\bm{X}), y) \bigg(\sqrt{p(y|\bm{X})} + \sqrt{q(y|\bm{X})} \bigg)  \bigg|\sqrt{p(y|\bm{X})} - \sqrt{q(y|\bm{X})} \bigg|  \bigg] \notag \\
    &\le  \Bigg\{ \mathbb{E}_{p(\bm{x}) p(y|\bm{x})}[(l(f(\bm{X}),Y))^2 ] \times \mathbb{E}_{p(\bm{x})}\bigg[\sum_{y\in\mathcal{Y}} \bigg(\sqrt{p(y|\bm{X})} - \sqrt{q(y|\bm{X})} \bigg)^2 \bigg]   \Bigg\}^\frac{1}{2} \notag \\
    &~~~+ \Bigg\{ \mathbb{E}_{p(\bm{x}) q(y|\bm{x})}[(l(f(\bm{X}),Y))^2 ] \times \mathbb{E}_{p(\bm{x})}\bigg[\sum_{y\in\mathcal{Y}} \bigg(\sqrt{p(y|\bm{X})} - \sqrt{q(y|\bm{X})} \bigg)^2 \bigg]   \Bigg\}^\frac{1}{2} \notag \\
    &\le \bigg( \sqrt{R_{l,p}(f)} + \sqrt{R_{l,q}(f)}  \bigg) \Bigg\{ 2M_l \mathbb{E}_{p(\bm{x})}[ D_{\mathrm{H}}^2(p(Y|\bm{X}), q(Y|\bm{X}))] \Bigg\}^\frac{1}{2} \notag \\
    &\le \bigg( \sqrt{R_{l,p}(f)} + \sqrt{R_{l,q}(f)}  \bigg) \Bigg\{ 2M_l D_{\mathrm{KL}}(p(Y|\bm{X}) || q(Y|\bm{X})) \Bigg\}^\frac{1}{2}.
\end{align}
The first inequality follows from the triangle inequality, and the second inequality applies the Cauchy-Schwarz inequality.
The third inequality exploits the fact that the function $x \mapsto x^2$ is $2M_l$-Lipschitz continuous over the interval $[0, M_l]$, leading to the bound:
\begin{align}
    \mathbb{E}_{p(\bm{x})p(y|\bm{x})} \Big[ \big(l(f(\bm{X}),Y) \big)^2 - 0^2 \Big] 
    \le 2 M_l \mathbb{E}_{p(\bm{x})p(y|\bm{x})}[l(f(\bm{X}),Y)] = 2 M_l R_{l,p}(f), \notag \\
    \mathbb{E}_{p(\bm{x})q(y|\bm{x})} \Big[ \big(l(f(\bm{X}),Y) \big)^2 - 0^2 \Big] 
    \le 2 M_l \mathbb{E}_{p(\bm{x})q(y|\bm{x})}[l(f(\bm{X}),Y)] = 2 M_l R_{l,q}(f). \notag \\
\end{align}
Here, $D_{\mathrm{H}}$ denotes the Hellinger distance.
The final inequality uses the well-known relation between the Hellinger distance and the KL divergence.

Therefore, we obtain:
\begin{align}
    |R_{l,p}(f)  - R_{l,q}(f)|
    &\le \bigg( \sqrt{R_{l,p}(f)} + \sqrt{R_{l,q}(f)}  \bigg) \Bigg\{ 2M_l \mathbb{E}_{p(\bm{x})}[ D_{\mathrm{H}}^2(p(Y|\bm{X}), q(Y|\bm{X}))] \Bigg\}^\frac{1}{2} \label{eq:diff_bound_p_q_hellinger} \\
    &\le \bigg( \sqrt{R_{l,p}(f)} + \sqrt{R_{l,q}(f)}  \bigg) \Bigg\{ 2M_l D_{\mathrm{KL}}(p(Y|\bm{X}) || q(Y|\bm{X})) \Bigg\}^\frac{1}{2}. \label{eq:diff_bound_p_q_kl}
\end{align}

Finally, dividing both sides of Eqs.~\eqref{eq:diff_bound_p_q_hellinger} and \eqref{eq:diff_bound_p_q_kl} by $\sqrt{R_{l,p}(f)} + \sqrt{R_{l,q}(f)}$ yields the result, thereby proving Lemma~\ref{lem:risk_div_inequality}:
\begin{align}
    \Big|\sqrt{R_{l,p}(f)}  - \sqrt{R_{l,q}(f)} \Big|
    &\le \Big\{ 2M_l \mathbb{E}_{p(\bm{x})}[ D_{\mathrm{H}}^2(p(Y|\bm{X}), q(Y|\bm{X}))] \Big\}^\frac{1}{2} \label{eq:diff_bound_p_q_hellinger_sqrt} \\
    &\le  \Big\{ 2M_l D_{\mathrm{KL}}(p(Y|\bm{X}) || q(Y|\bm{X})) \Big\}^\frac{1}{2}. \label{eq:diff_bound_p_q_kl_sqrt}
\end{align}
    
\end{proof}

\subsection{Preparation for proof of Theorem~\ref{thm:error_bound}}
\label{apdx_subsec:prepareation_proof_thm:error_bound}

To prepare for the proof of Theorem~\ref{thm:error_bound}, we establish several lemmas.

\begin{lemma}
    For any measurable $f \in \mathcal{F}$, any loss function $l$ that is bounded ($l \le M_l$) and $L_l$-Lipschitz continuous, and any $\delta \in (0,1)$, the following inequality holds with probability at least $1 - \delta$:
    \begin{align}
        |R_{l}(f) - \widehat{R}_{l, S_*}(f)|
        &\le 2 L_l \mathfrak{R}_{n}(\mathcal{F}) + M_l\sqrt{\frac{\log(1/\delta)}{2n}}.
        \label{eq:diff_bound_true_best_rademacher}
    \end{align}

\label{lem:diff_bound_true_best_rademacher}
\end{lemma}

\begin{proof}[Proof of Lemma \ref{lem:diff_bound_true_best_rademacher}]
Define the hypothesis class:
$$\mathcal{G} := \{\bm{x} \mapsto \mathbb{E}_{p_*(y|\bm{x})}[l(f(\bm{x}), Y)] | f \in \mathcal{F} \}.$$
Then, for any $f \in \mathcal{F}$, the following holds:
\begin{align}
    |R_{l}(f) - \widehat{R}_{l, S_*}(f)|
    &= |\mathbb{E}_{p_{*}(\bm{x},y)}[l(f(\bm{X}), Y)] - \mathbb{E}_{p_{S}(\bm{x}) p_{*}(y|\bm{x})}[l(f(\bm{X}), Y)]| \notag \\
    &\le \sup_{g \in \mathcal{G}} \bigg| \mathbb{E}_{p_{*}(\bm{x})}[g(\bm{X})] - \mathbb{E}_{p_{S}(\bm{x})}[g(\bm{X})] \bigg|.
\end{align}
Since $p_S(\bm{x})$ is the empirical distribution obtained from samples drawn independently from $p_*(\bm{x})$, the uniform law of large numbers \cite{mohri2018foundations} ensures that, with probability at least $1 - \delta$: 
\begin{align}
    |R_{l}(f) - \widehat{R}_{l, S_*}(f)|
    &\le 2 \mathfrak{R}_{n}(\mathcal{G}) + M_l\sqrt{\frac{\log(1/\delta)}{2n}}.
\end{align}
Because $p_{*}(y|\bm{x})$ is a discrete distribution, the expectation $\mathbb{E}_{p_{*}(y|\bm{x})}[l(f(\bm{x}), Y)]$ is a weighted sum of $l(f(\bm{x}), Y)$.
Also, $l$ is $L_l$-Lipschitz continuous.
From these facts, the mapping $y' \mapsto \mathbb{E}_{p_*(y|\bm{x})}[l(y', Y)]$ is also $L_l$-Lipschitz continuous.
Therefore, by Talagrand's lemma \cite{mohri2018foundations}, we obtain:
\begin{align}
    \mathfrak{R}_{n}(\mathcal{G}) \le L_l \mathfrak{R}_{n}(\mathcal{F}),
\end{align}
and thus, for any $\delta \in (0,1)$, the following inequality holds with probability at least $1 - \delta$:
\begin{align}
    |R_{l}(f) - \widehat{R}_{l, S_*}(f)|
    &\le 2 L_l \mathfrak{R}_{n}(\mathcal{F}) + M_l\sqrt{\frac{\log(1/\delta)}{2n}}. \notag
\end{align}

\end{proof}

From Lemma~\ref{lem:risk_div_inequality}, we obtain the following result:

\begin{lemma}
    Let $p(\bm{x})$ be any probability distribution over $\mathcal{X}$, and let $p(y|\bm{x})$ and $q(y|\bm{x})$ be any conditional distributions over $\mathcal{Y}$.
    Then, for any bounded loss function $l \le M_l$ and any measurable $f: \mathcal{X} \to \mathcal{Y}$, the difference between the two risks 
    \begin{align}
        \textstyle R_{l,p}(f):= \mathbb{E}_{p(\bm{x})p(y|\bm{x})}[l(f(\bm{X}),Y)], ~~~ R_{l,q}(f):= \mathbb{E}_{p(\bm{x})q(y|\bm{x})}[l(f(\bm{X}),Y)] \notag
    \end{align}
    is bounded as follows:
    \begin{align}
    \begin{split}
        \bigg|\sqrt{R_{l, p}(f)} - \sqrt{R_{l ,q}(f)} \bigg|
        &\le \Bigg\{2\sqrt{R_{l, p}(f)} +  \Bigg( 2M_l \mathbb{E}_{p(\bm{x})}[ D_{\mathrm{H}}^2(p(Y|\bm{X}), q(Y|\bm{X}))] \Bigg)^\frac{1}{2} \Bigg\}  \\
        &~~~~~~~~\times \Bigg\{ 2M_l \mathbb{E}_{p(\bm{x})}[ D_{\mathrm{H}}^2(p(Y|\bm{X}), q(Y|\bm{X}))] \Bigg\}^\frac{1}{2},
    \end{split}
    \label{eq:diff_bound_p_q_hellinger_pred_ver}
    \end{align}

    \begin{align}
    \begin{split}
        \bigg|\sqrt{R_{l, p}(f)} - \sqrt{R_{l ,q}(f)} \bigg|
        &\le \Bigg\{2\sqrt{R_{l, p}(f)} +  \Bigg( 2M_l D_{\mathrm{KL}}(p_*(Y|\bm{X}) || p_{\lambda}(Y|\bm{X})) \Bigg)^\frac{1}{2} \Bigg\} \\
        &~~~~~~~~\times \Bigg\{ 2M_l D_{\mathrm{KL}}(p_*(Y|\bm{X}) || p_{\lambda}(Y|\bm{X})) \Bigg\}^\frac{1}{2}. 
    \end{split}
    \label{eq:diff_bound_p_q_kl_pred_ver}
    \end{align}
    
\label{lem:risk_div_inequality_sqrt_ver}
\end{lemma}

\begin{proof}[Proof of Lemma~\ref{lem:risk_div_inequality_sqrt_ver}]

From Eq.~\eqref{eq:diff_bound_p_q_hellinger_sqrt} in the proof of Lemma~\ref{lem:risk_div_inequality}, we obtain the following bound for any probability distribution $p(\bm{x})$ over $\mathcal{X}$, any conditional distributions $p(y|\bm{x})$ and $q(y|\bm{x})$ over $\mathcal{Y}$, any bounded loss function $l \le M_l$, and any measurable $f: \mathcal{X} \rightarrow \mathcal{Y}$:
\begin{align}
    \sqrt{R_{l,p}(f)}
    &= \sqrt{R_{l,q}(f)} - \sqrt{R_{l,q}(f)} + \sqrt{R_{l,p}(f)} \notag \\
    &\le \sqrt{R_{l,q}(f)} + \bigg|\sqrt{R_{l,p}(f)} - \sqrt{R_{l,q}(f)} \bigg| \notag \\
    &\le \sqrt{R_{l,q}(f)} + \Big\{ 2M_l \mathbb{E}_{p(\bm{x})}[ D_{\mathrm{H}}^2(p_*(Y|\bm{X}), p_{\lambda}(Y|\bm{X}))] \Big\}^\frac{1}{2}.
\end{align}
The first inequality uses the basic inequality for absolute values, and the second inequality follows directly from Eq.~\eqref{eq:diff_bound_p_q_hellinger_sqrt}.

By applying the above result to the right-hand side of Eq.~\eqref{eq:diff_bound_p_q_hellinger} in the proof of Lemma~\ref{lem:risk_div_inequality}, the following holds:
\begin{align}
\begin{split}
    |R_{l, p}(f)  - R_{l, q}(f)| 
        &\le \Bigg\{2\sqrt{R_{l, q}(f)} +  \Big( 2M_l \mathbb{E}_{p(\bm{x})}[ D_{\mathrm{H}}^2(p_*(Y|\bm{X}), p_{\lambda}(Y|\bm{X}))] \Big)^\frac{1}{2} \Bigg\} \notag \\
        &~~~~~~~~\times \Bigg\{ 2M_l \mathbb{E}_{p(\bm{x})}[ D_{\mathrm{H}}^2(p(Y|\bm{X}), q(Y|\bm{X}))] \Bigg\}^\frac{1}{2}.
\end{split}
\end{align}

Furthermore, by using Eqs.~\eqref{eq:diff_bound_p_q_kl} and \eqref{eq:diff_bound_p_q_kl_sqrt}, we can replace $D_{\mathrm{H}}^2$ with the KL divergence $D_{\mathrm{KL}}$, yielding the analogous bound:
\begin{align}
\begin{split}
    |R_{l, p}(f)  - R_{l, q}(f)|  
        &\le \Bigg\{2\sqrt{R_{l, q}(f)} +  \Big( 2M_l D_{\mathrm{KL}}(p(Y|\bm{X}) || q(Y|\bm{X})) \Big)^\frac{1}{2} \Bigg\} \notag \\
        &~~~~~~~~\times \Big\{ 2M_l D_{\mathrm{KL}}(p(Y|\bm{X}) || q(Y|\bm{X})) \Big\}^\frac{1}{2}.
\end{split}
\end{align}
    
\end{proof}

In Lemma~\ref{lem:risk_div_inequality_sqrt_ver}, by setting $p(\bm{x}) = p_S(\bm{x})$, $p(y|\bm{x}) = p_{*}(y|\bm{x})$, and $q(y|\bm{x}) = p_{\lambda}(y|\bm{x})$, we obtain $R_{l,p}(f)=\widehat{R}_{l, S_*}(f)$ and $R_{l,q}(f) = \widehat{R}_{l, \widehat{S}}(f)$. 
Consequently, the following corollary holds:

\begin{corollary}
    For any measurable $f \in \mathcal{F}$, the following inequalities hold:
    \begin{align}
    \begin{split}
        |\widehat{R}_{l, S_*}(f)  - \widehat{R}_{l, \widehat{S}}(f)| 
            &\le \Bigg\{2\sqrt{\widehat{R}_{l, \widehat{S}}(f)} +  \Bigg( 2M_l \frac{1}{n}\sum^n_{i=1} D_{\mathrm{H}}^2(p_*(Y|\bm{x}_i), p_{\lambda}(Y|\bm{x}_i)) \Bigg)^\frac{1}{2} \Bigg\}  \\
            &~~~~~~~~\times \Bigg\{ 2M_l \frac{1}{n}\sum^n_{i=1} D_{\mathrm{H}}^2(p_*(Y|\bm{x}_i), p_{\lambda}(Y|\bm{x}_i)) \Bigg\}^\frac{1}{2}, 
    \end{split}
    \label{eq:diff_bound_best_pred_hellinger_pred_ver}
    \end{align}
    
    \begin{align}
    \begin{split}
        |\widehat{R}_{l, S_*}(f)  - \widehat{R}_{l, \widehat{S}}(f)| 
            &\le \Bigg\{2\sqrt{\widehat{R}_{l, \widehat{S}}(f)} +  \Bigg( 2M_l \frac{1}{n}\sum^n_{i=1} D_{\mathrm{KL}}(p_*(Y|\bm{x}_i) || p_{\lambda}(Y|\bm{x}_i)) \Bigg)^\frac{1}{2} \Bigg\} \\
            &~~~~~~~~\times \Bigg\{ 2M_l \frac{1}{n}\sum^n_{i=1} D_{\mathrm{KL}}(p_*(Y|\bm{x}_i) || p_{\lambda}(Y|\bm{x}_i)) \Bigg\}^\frac{1}{2}. 
    \end{split}
    \label{eq:diff_bound_best_pred_kl_pred_ver}
    \end{align}

\label{cor:diff_bound_best_pred_pred_ver}
\end{corollary}

\subsection{Proof of Theorem~\ref{thm:error_bound}}
\label{apdx_subsec:proof_thm:error_bound}

Using the above lemmas and corollary, we now prove Theorem~\ref{thm:error_bound} as follows.

\begin{proof}[Proof of Theorem~\ref{thm:error_bound}]

The left-hand side of Eq.~\eqref{eq:error_bound} in Theorem~\ref{thm:error_bound} can be upper-bounded as:
\begin{align}
    R_{l}(f_{\widehat{S}}) 
    &=R_{l}(f_{\widehat{S}}) - \widehat{R}_{l, \widehat{S}}(f_{\widehat{S}}) + \widehat{R}_{l, \widehat{S}}(f_{\widehat{S}}) -  R_{l}(f_{\mathcal{F}}) + R_{l}(f_{\mathcal{F}}) \notag \\
    &\le R_{l}(f_{\widehat{S}}) - \widehat{R}_{l, \widehat{S}}(f_{\widehat{S}}) + \widehat{R}_{l, \widehat{S}}(f_{\mathcal{F}}) -  R_{l}(f_{\mathcal{F}}) + R_{l}(f_{\mathcal{F}})\notag \\
    &\le \underset{\text{(a1)}}{\underbrace{|R_{l}(f_{\widehat{S}}) - \widehat{R}_{l, \widehat{S}}(f_{\widehat{S}})|}} 
                           + \underset{\text{(a2)}}{\underbrace{|\widehat{R}_{l, \widehat{S}}(f_{\mathcal{F}}) -  R_{l}(f_{\mathcal{F}})| }} + R_{l}(f_{\mathcal{F}}). 
\label{eq:error_bound_tmp1}
\end{align}
The first inequality follows from the definition of $f_{\widehat{S}}$ as the empirical risk minimizer, i.e., $\widehat{R}_{l,\widehat{S}}(f_{\widehat{S}}) \le \widehat{R}_{l,\widehat{S}}(f_{\mathcal{F}})$, and the second inequality follows from the basic property of absolute values.

For term (a1) in Eq.~\eqref{eq:error_bound_tmp1}, we have the following bound from Corollary~\ref{cor:diff_bound_best_pred_pred_ver}:
\begin{align}
    &|R_l(f_{\widehat{S}}) - \widehat{R}_{l, \widehat{S}}(f_{\widehat{S}})| \notag \\
    &= |R_{l}(f_{\widehat{S}}) - \widehat{R}_{l, S_*}(f_{\widehat{S}}) + \widehat{R}_{l, S_*}(f_{\widehat{S}})  - \widehat{R}_{l, \widehat{S}}(f_{\widehat{S}})| \notag \\
    &\le |R_{l}(f_{\widehat{S}}) - \widehat{R}_{l, S_*}(f_{\widehat{S}})| +  |\widehat{R}_{l, S_*}(f_{\widehat{S}})  - \widehat{R}_{l, \widehat{S}}(f_{\widehat{S}})|  \notag \\
    &\le \max_{f \in \mathcal{F}}|R_{l}(f) - \widehat{R}_{l, S_*}(f)| +  |\widehat{R}_{l, S_*}(f_{\widehat{S}})  - \widehat{R}_{l, \widehat{S}}(f_{\widehat{S}})|  \notag \\
    &\le \max_{f \in \mathcal{F}}|R_{l}(f) - \widehat{R}_{l, S_*}(f)|  + \Bigg\{2\sqrt{\widehat{R}_{l, \widehat{S}}(f_{\widehat{S}})} +  \Bigg( 2M_l \frac{1}{n}\sum^n_{i=1} D_{\mathrm{KL}}(p_*(Y|\bm{x}_i) || p_{\lambda}(Y|\bm{x}_i)) \Bigg)^\frac{1}{2} \Bigg\} \notag \\
        &~~~~~~~~\times \Bigg\{ 2M_l \frac{1}{n}\sum^n_{i=1} D_{\mathrm{KL}}(p_*(Y|\bm{x}_i) || p_{\lambda}(Y|\bm{x}_i)) \Bigg\}^\frac{1}{2}.  
\label{eq:error_bound_tmp_a1_1}
\end{align}
The first inequality uses the triangle inequality, and the second uses an inequality based on $\max$.
The third inequality uses Corollary~\ref{cor:diff_bound_best_pred_pred_ver}.
Moreover, $\widehat{R}_{l,\widehat{S}}(f_{\widehat{S}})$ in Eq.~\eqref{eq:error_bound_tmp_a1_1} can be upper-bounded as follows:
\begin{align}
    \widehat{R}_{l, \widehat{S}}(f_{\widehat{S}})
    &= \widehat{R}_{l, \widehat{S}}(f_{\widehat{S}}) - R_{l}(f_{\mathcal{F}}) + R_{l}(f_{\mathcal{F}}) \notag \\
    &\le \widehat{R}_{l, \widehat{S}}(f_{\mathcal{F}}) - R_{l}(f_{\mathcal{F}}) + R_{l}(f_{\mathcal{F}}) \notag \\
    &\le |\widehat{R}_{l, \widehat{S}}(f_{\mathcal{F}}) - R_{l}(f_{\mathcal{F}}) | + R_{l}(f_{\mathcal{F}}).
\label{eq:error_bound_tmp_a1_2}
\end{align}
The first inequality follows from the definition of $f_{\widehat{S}}$, i.e., $\widehat{R}_{l,\widehat{S}}(f_{\widehat{S}}) \le \widehat{R}_{l,\widehat{S}}(f_{\mathcal{F}})$, and the second inequality follows from the absolute value property.
Note that the first term on the right-hand side of Eq.~\eqref{eq:error_bound_tmp_a1_2} is the same as term (a2) in Eq.~\eqref{eq:error_bound_tmp1}.

Next, we consider term (a2) in Eq.~\eqref{eq:error_bound_tmp1}.
We begin with the following inequality:
\begin{align}
    |\widehat{R}_{l, \widehat{S}}(f_{\mathcal{F}}) - R_{l}(f_{\mathcal{F}}) |
    \le \max_{f \in \mathcal{F}}|R_{l}(f) - \widehat{R}_{l, \widehat{S}}(f)|.
\label{eq:error_bound_tmp_a2_1}
\end{align}
Defining $f' := \argmax_{f \in \mathcal{F}} |R_l(f) - \widehat{R}_{l,\widehat{S}}(f)|$, the following holds:
\begin{align}
    \max_{f \in \mathcal{F}}|R_{l}(f) - \widehat{R}_{l, \widehat{S}}(f)|
    &=|R_{l}(f') - \widehat{R}_{l, \widehat{S}}(f')| \notag \\
    &= |R_{l}(f') - \hat{R}_{l, S_*}(f') +  \widehat{R}_{l, S_*}(f')  - \widehat{R}_{l, \widehat{S}}(f')| \notag \\
    &\le |R_{l}(f') - \hat{R}_{l, S_*}(f')| +  |\widehat{R}_{l, S_*}(f')  - \widehat{R}_{l, \widehat{S}}(f')| \notag \\
    &\le \max_{f \in \mathcal{F}}|R_{l}(f) - \hat{R}_{l, S_*}(f)| +  |\widehat{R}_{l, S_*}(f')  - \widehat{R}_{l, \widehat{S}}(f')|.
\end{align}
The first inequality uses the triangle inequality, and the second uses the inequality involving maximum.
Applying this result to Eq.~\eqref{eq:error_bound_tmp_a2_1}, term (a2) can be upper-bounded as:
\begin{align}
    |\widehat{R}_{l, \widehat{S}}(f_{\mathcal{F}}) - R_{l}(f_{\mathcal{F}}) |
    \le \max_{f \in \mathcal{F}}|R_{l}(f) - \hat{R}_{l, S_*}(f)| +  |\widehat{R}_{l, S_*}(f')  - \widehat{R}_{l, \widehat{S}}(f')|.
\label{eq:error_bound_tmp_a2_2}
\end{align}
The second term on the right-hand side of Eq.~\eqref{eq:error_bound_tmp_a2_2} can be upper-bounded as follows:
\begin{align}
    &|\widehat{R}_{l, S_*}(f')  - \widehat{R}_{l, \widehat{S}}(f')|  \notag \\
    &= | \mathbb{E}_{p_{S}(\bm{x}) p_{*}(y|\bm{x})}[l(f'(\bm{X}), Y)]  - \mathbb{E}_{p_{S}(\bm{x}) p_{\lambda}(y|\bm{x})}[l(f'(\bm{X}), Y)] | \notag \\
    &\le \mathbb{E}_{p_{S}(\bm{x})}\bigg[\sum_{y \in \mathcal{Y}} l(f'(\bm{X}), y) |p_{*}(y|\bm{X}) - p_{\lambda}(y|\bm{X}) | \bigg] \notag \\
    &\le M_l \mathbb{E}_{p_{S}(\bm{x})}\bigg[\sum_{y \in \mathcal{Y}} |p_{*}(y|\bm{X}) - p_{\lambda}(y|\bm{X}) |  \bigg] \notag \\
    &\le M_l \frac{1}{n}\sum^n_{i=1} \bigg\{ 2D_{\mathrm{KL}}(p_*(Y|\bm{x}_i) || p_{\lambda}(Y|\bm{x}_i))  \bigg\}^{\frac{1}{2}} \notag \\ 
    &\le M_l \bigg\{ \frac{2}{n}\sum^n_{i=1} D_{\mathrm{KL}}(p_*(Y|\bm{x}_i) || p_{\lambda}(Y|\bm{x}_i))  \bigg\}^{\frac{1}{2}}, 
\end{align}
where the first inequality uses the triangle inequality, the second uses the boundedness $l \le M_l$, the third applies Pinsker’s inequality, and the fourth uses Jensen’s inequality.

By applying this to Eq.~\eqref{eq:error_bound_tmp_a2_2}, term (a2) in Eq.~\eqref{eq:error_bound_tmp1} can be upper-bounded as follows:
\begin{align}
    &|\widehat{R}_{l, \widehat{S}}(f_{\mathcal{F}}) - R_{l}(f_{\mathcal{F}}) |\notag \\
    &\le \max_{f \in \mathcal{F}}|R_{l}(f) - \hat{R}_{l, S_*}(f)| +  M_l \bigg\{ \frac{2}{n}\sum^n_{i=1} D_{\mathrm{KL}}(p_*(Y|\bm{x}_i) || p_{\lambda}(Y|\bm{x}_i))  \bigg\}^{\frac{1}{2}}.
\label{eq:error_bound_tmp_a2_3}
\end{align}
Applying this result to Eq.~\eqref{eq:error_bound_tmp_a1_2}, the following holds:
\begin{align}
    \widehat{R}_{l, \widehat{S}}(f_{\widehat{S}})
    &\le \max_{f \in \mathcal{F}}|R_{l}(f) - \hat{R}_{l, S_*}(f)| +  M_l \bigg\{ \frac{2}{n}\sum^n_{i=1} D_{\mathrm{KL}}(p_*(Y|\bm{x}_i) || p_{\lambda}(Y|\bm{x}_i))  \bigg\}^{\frac{1}{2}} + R_{l}(f_{\mathcal{F}}).
\end{align}
Furthermore, applying this result to Eq.~\eqref{eq:error_bound_tmp_a1_1}, term (a1) in Eq.~\eqref{eq:error_bound_tmp1} can be upper-bounded as follows:
\begin{align}
\begin{split}
    &|R_l(f_{\widehat{S}}) - \widehat{R}_{l, \widehat{S}}(f_{\widehat{S}})| \\
    &\le \max_{f \in \mathcal{F}}|R_{l}(f) - \widehat{R}_{l, S_*}(f)| \\
    &+ \Bigg\{2\Bigg( \max_{f \in \mathcal{F}}|R_{l}(f) - \hat{R}_{l, S_*}(f)| +  M_l \bigg\{ \frac{2}{n}\sum^n_{i=1} D_{\mathrm{KL}}(p_*(Y|\bm{x}_i) || p_{\lambda}(Y|\bm{x}_i))  \bigg\}^{\frac{1}{2}} + R_{l}(f_{\mathcal{F}}) \Bigg)^{\frac{1}{2}}  \\
    &~~~~~~~~~+\Bigg( 2M_l \frac{1}{n}\sum^n_{i=1} D_{\mathrm{KL}}(p_*(Y|\bm{x}_i) || p_{\lambda}(Y|\bm{x}_i)) \Bigg)^\frac{1}{2} \Bigg\} \\ 
    &~~~~~~~~~~~\times \Bigg\{ 2M_l \frac{1}{n}\sum^n_{i=1} D_{\mathrm{KL}}(p_*(Y|\bm{x}_i) || p_{\lambda}(Y|\bm{x}_i)) \Bigg\}^\frac{1}{2}.  
\end{split}
\label{eq:error_bound_tmp_a1_max}
\end{align}

Thus, applying Eqs.~\eqref{eq:error_bound_tmp_a1_max} and \eqref{eq:error_bound_tmp_a2_3} to terms (a1) and (a2) in Eq.~\eqref{eq:error_bound_tmp1}, respectively, the following holds:
\begin{align}
\begin{split}
    R_{l}(f_{\widehat{S}}) 
    &\le 2\max_{f \in \mathcal{F}}|R_{l}(f) - \widehat{R}_{l, S_*}(f)| \\
    &~~~+ \Bigg\{2\Bigg\{\max_{f \in \mathcal{F}}|R_{l}(f) - \hat{R}_{l, S_*}(f)| \\
    &~~~~~~~~~~~~~~~~~~~~~~~~+  M_l \bigg( \frac{2}{n}\sum^n_{i=1} D_{\mathrm{KL}}(p_*(Y|\bm{x}_i) || p_{\lambda}(Y|\bm{x}_i))  \bigg)^{\frac{1}{2}} + R_{l}(f_{\mathcal{F}}) \Bigg\}^{\frac{1}{2}}  \\
    &~~~~~~~~~~~~~~~~+\Bigg( 2M_l \frac{1}{n}\sum^n_{i=1} D_{\mathrm{KL}}(p_*(Y|\bm{x}_i) || p_{\lambda}(Y|\bm{x}_i)) \Bigg)^\frac{1}{2} \Bigg\} \\
    &~~~~~~~~~~~~~~~~~~~\times \Bigg\{ 2M_l \frac{1}{n}\sum^n_{i=1} D_{\mathrm{KL}}(p_*(Y|\bm{x}_i) || p_{\lambda}(Y|\bm{x}_i)) \Bigg\}^\frac{1}{2} \\
    &~~~+  M_l \bigg\{ \frac{2}{n}\sum^n_{i=1} D_{\mathrm{KL}}(p_*(Y|\bm{x}_i) || p_{\lambda}(Y|\bm{x}_i))  \bigg\}^{\frac{1}{2}} + R_{l}(f_{\mathcal{F}}).
\end{split}
\label{eq:error_bound_tmp2}
\end{align}

Finally, applying Lemma~\ref{lem:diff_bound_true_best_rademacher} to Eq.~\eqref{eq:error_bound_tmp2}, for any $\delta \in (0,1)$, the following holds with probability at least $1 - \delta$:
\begin{align}
    R_{l}(f_{\widehat{S}}) 
    &\le  4 L_l \mathfrak{R}_{n}(\mathcal{F}) + 2M_l\sqrt{\frac{\log(1/\delta)}{2n}} \notag \\
    &+ 2\Bigg\{ 2 L_l \mathfrak{R}_{n}(\mathcal{F}) + M_l\sqrt{\frac{\log(1/\delta)}{2n}} + R_l(f_{\mathcal{F}}) \notag \\
    &~~~~~~~~~~~~~~~~~~~~~~~+  M_l \bigg( \frac{2}{n}\sum^n_{i=1} D_{\mathrm{KL}}(p_*(Y|\bm{x}_i) || p_{\lambda}(Y|\bm{x}_i))  \bigg)^{\frac{1}{2}} \Bigg\}^{\frac{1}{2}} \notag \\
    &~~~~~~~~~~~~~~~\times   \Bigg\{ 2M_l \frac{1}{n}\sum^n_{i=1} D_{\mathrm{KL}}(p_*(Y|\bm{x}_i)|| p_{\lambda}(Y|\bm{x}_i)) \Bigg\}^\frac{1}{2} \notag \\
    &+ M_l \Bigg\{ \frac{2}{n}\sum^n_{i=1} D_{\mathrm{KL}}(p_*(Y|\bm{x}_i) || p_{\lambda}(Y|\bm{x}_i)) + \bigg( \frac{2}{n}\sum^n_{i=1} D_{\mathrm{KL}}(p_*(Y|\bm{x}_i) || p_{\lambda}(Y|\bm{x}_i))  \bigg)^{\frac{1}{2}} \Bigg\} \notag \\
    &+ R_l(f_{\mathcal{F}}). \notag 
\end{align}

Therefore, Theorem~\ref{thm:error_bound} is proved.

\end{proof}

\section{Detail of Experimental Settings}
\label{apdx:sec_detal_exp_settings}

\subsection{Datasets}
\label{apdx:subsec_dataset}

In this experiment, we used four image classification datasets-MNIST \cite{mnist}, fashion-MNIST \cite{fashion_mnist}, kuzushiji-MNIST \cite{kuzushiji_mnist}, and CIFAR-10 \cite{cifar10}, as well as four datasets from the UCI repository \cite{Dua:2019}: HAR \cite{har}, Letter \cite{letter}, Optdigits \cite{optical}, and Pendigits \cite{pendigits}.
The full names of the UCI datasets are as follows: HAR: ``Human Activity Recognition Using Smartphones'', Letter: ``Letter Recognition'', Optdigits: ``Optical Recognition of Handwritten Digits'', Pendigits: ``Pen-Based Recognition of Handwritten Digits''.
Table~\ref{tab:outline_datasets} summarizes the key information for each dataset.
Since the feature scales vary across dimensions in the UCI datasets, we normalized all features to fall within the range $[0,1]$.
An overview of each dataset is provided below:

\begin{table}[h]
    \centering
    \caption{Outline of datasets.
             }
    \begin{tabular}{c|c|c|c}
        \hline
         dataset         & data size  &  input dimension   &  classes \\ \hline \hline
         MNIST           & 60000 & 784    &  10  \\ \hline 
         fashion-MNIST   & 60000 & 784    &  10 \\ \hline
         kuzushiji-MNIST & 60000 & 784    &  10 \\ \hline
         CIFAR-10        & 60000 & 3072   &  10 \\ \hline
         HAR             & 10299 & 562    &  8  \\ \hline
         Letter          & 20000 & 16     &  26 \\ \hline
         Optdigits         & 5620  & 64     &  10 \\ \hline
         Pendigits       & 10992 & 16     &  10 \\ \hline
    \end{tabular}
    \label{tab:outline_datasets}
\end{table}

\textbf{MNIST} \cite{mnist}: 
A grayscale 32×32 pixel image classification dataset of handwritten digits from 0 to 9.
The dataset is available at \href{http://yann.lecun.com/exdb/mnist/}{http://yann.lecun.com/exdb/mnist/}.

\textbf{fashion-MNIST} \cite{fashion_mnist}: 
A grayscale 32×32 pixel image classification dataset of fashion items such as T-shirts and sneakers.
The dataset is available at \href{https://github.com/zalandoresearch/fashion-mnist}{https://github.com/zalandoresearch/fashion-mnist}.

\textbf{kuzushiji-MNIST} \cite{kuzushiji_mnist}:
A grayscale 32×32 pixel image classification dataset of handwritten Japanese characters.
The dataset is available at \href{https://github.com/rois-codh/kmnist}{https://github.com/rois-codh/kmnist}.

\textbf{CIFAR-10} \cite{cifar10}:
A color image classification dataset of objects such as cars and cats.
The dataset is available at \href{https://www.cs.toronto.edu/~kriz/cifar.html}{https://www.cs.toronto.edu/~kriz/cifar.html}.

\textbf{HAR} \cite{har}:
A dataset for classifying activities such as walking, based on accelerometer and gyroscope data from smartphones of 30 individuals.
The dataset is available at \href{https://archive.ics.uci.edu/dataset/240/human+activity+recognition+using+smartphones}{https://archive.ics.uci.edu/dataset/240/human+activity+recognition+using+smartphones}

\textbf{Letter} \cite{letter}:
A dataset for classifying alphabets based on 16-dimensional numerical features of handwritten letters.
The dataset is available at \href{https://archive.ics.uci.edu/dataset/59/letter+recognition}{https://archive.ics.uci.edu/dataset/59/letter+recognition}

\textbf{Optdigits} \cite{optical}:
A grayscale 8×8 image classification dataset of handwritten digits.
The dataset is available at \href{https://archive.ics.uci.edu/dataset/80/optical+recognition+of+handwritten+digits}{https://archive.ics.uci.edu/dataset/80/optical+recognition+of+handwritten+digits}

\textbf{Pendigits} \cite{pendigits}:
A dataset for classifying digits based on 16-dimensional feature vectors of pen-written handwritten numbers.
The dataset is available at \href{https://archive.ics.uci.edu/dataset/81/pen+based+recognition+of+handwritten+digits}{https://archive.ics.uci.edu/dataset/81/pen+based+recognition+of+handwritten+digits}

\subsection{Label Generation Models}
\label{apdx:subsec_label_generation}

For all datasets, a neural network with a single hidden layer of width $5000$ was adopted as the label generation model for each dataset.
For image datasets including MNIST, fashion-MNIST, kuzushiji-MNIST, and CIFAR-10, the model was trained using $50000$ training samples.
For the UCI datasets, $90$\% of the entire dataset was randomly sampled and used for training.
The supervision for this training was based on the original hard labels provided in each dataset.
Cross-entropy loss was employed as the loss function.
The optimization of each label generation model was performed using Adam \cite{kingma2014adam}, with the following hyperparameter settings: a learning rate of $0.0005$, batch size of $512$, and a weight decay of $0.0002$.
The number of training epochs was set to $20$ for MNIST, fashion-MNIST, and kuzushiji-MNIST, and to $100$ for the remaining datasets.
To prevent the label generation models from producing overly confident predictions, label smoothing \cite{Szegedy_2016_CVPR,muller2019does} was applied with a confidence level of $0.9$ for the hard-labeled classes.

Table~\ref{tab:outline_label_generation_model} summarizes the information about the label generation model for each dataset.
Figures~\ref{fig:label_gen_conf_image} and \ref{fig:label_gen_conf_uci} show the confidence distributions of the outputs produced by the label generation models.
In these figures, the maximum predicted probabilities were collected, and histograms of these values were plotted.
These visualizations help quantify the uncertainty involved in the label assignments produced by the label generation models for each dataset.
By comparing Figures~\ref{fig:label_gen_conf_image} and \ref{fig:label_gen_conf_uci} with Figure~\ref{fig:compare_size_acc_main}, which evaluates the performance of classification models using T1OC and T2OC, it can be observed that T1OC and T2OC remain effective regardless of variations in the confidence levels of the label generation model.

\begin{table}[!h]
    \centering
    \caption{Outline of label generation models. 
             Each evaluation metric is computed on the test data.
             Evaluation metrics other than accuracy are computed using macro-averaging.
             Here, the hard labels originally assigned to each dataset are used as the ground truth labels.
             For the four image classification datasets, $10000$ samples were used as test data, while for the four UCI datasets, $10$\% of the entire dataset was allocated for testing.
             }
    \begin{tabular}{c|c|c|c|c|c}
        \hline
         dataset         & Accuracy & F1 & Precision & Recall & AUROC \\ \hline \hline
         MNIST           & 0.9870   & 0.9870   &  0.9871   & 0.9870 & 0.9998 \\ \hline 
         fashion-MNIST   & 0.8930   & 0.8926   &  0.8928   & 0.8930 & 0.9912 \\ \hline
         kuzushiji-MNIST & 0.9297   & 0.9297   &  0.9307   & 0.9297 & 0.9948 \\ \hline
         CIFAR-10        & 0.5656   & 0.5641   &  0.5633   & 0.5656 & 0.9072 \\ \hline
         HAR             & 0.9893   & 0.9897   &  0.9896   & 0.9900 & 0.9998  \\ \hline
         Letter          & 0.8745   & 0.8738   &  0.8802   & 0.8742 & 0.9944 \\ \hline
         Optdigits       & 0.9929   & 0.9928   &  0.9923   & 0.9934 & 0.9999 \\ \hline
         Pendigits       & 0.9891   & 0.9890   & 0.9890    & 0.9890 & 0.9988 \\ \hline
    \end{tabular}
    \label{tab:outline_label_generation_model}
\end{table}

\begin{figure*}[!h]
    \centering
    \includegraphics[width=1.0\linewidth]{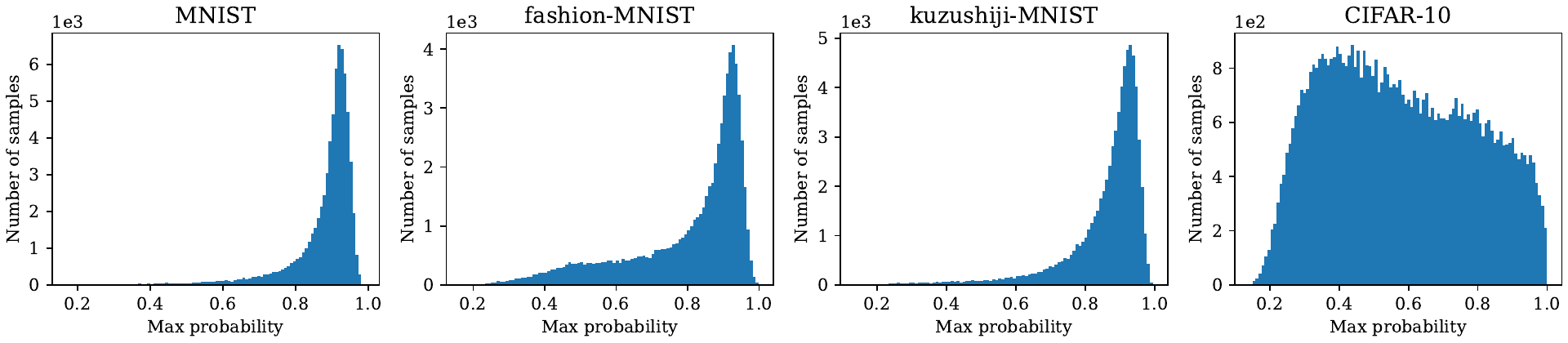}
    \caption{Max confidence distribution of label generation models (Image datasets)}
    \label{fig:label_gen_conf_image}
\end{figure*}

\begin{figure*}[!h]
    \centering
    \includegraphics[width=1.0\linewidth]{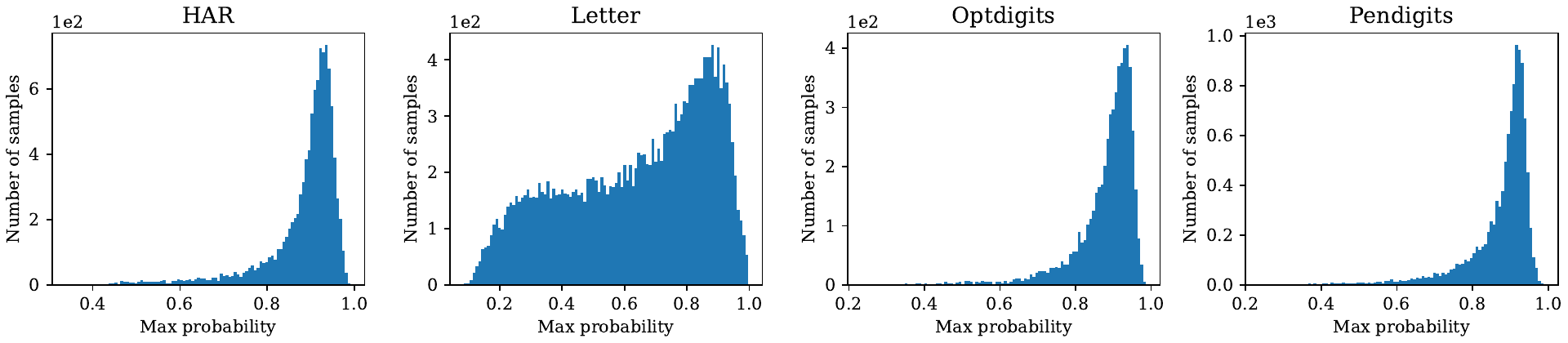}
    \caption{Max confidence distribution of label generation models (UCI datasets)}
    \label{fig:label_gen_conf_uci}
\end{figure*}

\subsection{Compute Resources}
\label{apdx:subsec_resources}

The experiments in this paper were conducted using the following computing environment:
OS: Rocky Linux 9.2
CPU: Intel(R) Xeon(R) Gold 6312U (24 cores, 48 threads, 2.40 GHz), 
RAM: 1.0 TB, 
SSD: 1.0 TB, 
GPU: NVIDIA RTX A6000 (48 GB VRAM), 
CUDA: 11.8, 
Python: 3.12.1.

\subsection{Detail of Compared Methods}
\label{apdx:subsec_detail_compared_method}

This section summarizes the compared methods and their specific configurations in our experiments.

\textbf{Label Smoothing (LS)} \cite{Szegedy_2016_CVPR}:
LS replaces hard labels, which are deterministic distributions, with soft labels by mixing in a uniform distribution with a small weight.
For example, if the confidence for the hard-labeled class is set to $0.9$, the remaining $0.1$ is uniformly distributed among the other classes to construct the soft label.

\textbf{Self-Distillation (SD)} \cite{furlanello2018born}:
SD involves training a teacher model on a dataset with hard labels and then using the output distributions of this teacher model as soft labels for training a student model.
It has been reported that the resulting student model can achieve better classification performance than the teacher model \cite{furlanello2018born}.
The training configuration for the student model is provided in ``Learning.'' of Section~\ref{subsec:exp_settings_and_result}.
The teacher model was trained with the same architecture, optimization method, and hyperparameter settings as the student model.

\textbf{Graph Laplacian Label Enhancement (GLLE)} \cite{xu2021label}:
GLLE is a Label Enhancement (LE) method that aims to recover label distributions by leveraging instance similarity information in the feature space.
It constructs soft labels using a fixed nonlinear transformation of the input features, followed by a linear transformation with a weight matrix $\bm{W}$, which is optimized during training.
To incorporate instance-level similarity information into the optimization, GLLE introduces regularization terms based on a similarity graph constructed as follows.
First, the similarity between instances is quantified using a Gaussian kernel.
A graph is then built by treating instances as graph nodes, connecting each instance to its $K$-nearest neighbors and assigning edge weights according to the computed similarities.
Using this graph, GLLE introduces a regularization term $\mathrm{Reg}_1$ that encourages similar instances to have similar soft label distributions.
Additionally, a second regularization term $\mathrm{Reg}_2$ incorporates label-wise similarity information.

GLLE has several hyperparameters that control the influence of these regularization terms and the graph construction: the regularization weights $\alpha_1 \in \mathbb{R}$ and $\alpha_2 \in \mathbb{R}$ for $\mathrm{Reg}_1$ and $\mathrm{Reg}_2$, respectively, the number of neighbors $K \in \mathbb{N}+$, and the scaling parameter $\sigma \in \mathbb{R}+$ for the Gaussian kernel.
In this experiment, following the setup of Xu et al. (2021), $K$ was set to the number of classes plus one, and $\sigma$ was set to $1$ \cite{xu2021label}.
Moreover, as the released code uses the same value for both regularization weights, the two were treated as a single hyperparameter $\alpha \in \mathbb{R}$, such that $\alpha_1 = \alpha_2 = \alpha$.
In Xu et al. (2021), $\alpha$ was tuned over the range $\{0.01, 0.1, \ldots, 100\}$ \cite{xu2021label}.
It was found that certain settings, such as $\alpha = 100$, degraded the performance of the resulting classification model.
Therefore, $\alpha = 0.01$ and $\alpha = 0.1$ were tested, and $\alpha = 0.01$ was adopted as it yielded better overall average performance across five trials.
Here, there was no significant difference in the classification performance between $\alpha = 0.01$ and $\alpha = 0.1$.
The original MATLAB implementation of GLLE is available at the following URL: \href{https://github.com/palm-ml/GLLE/tree/main}{https://github.com/palm-ml/GLLE/tree/main}.
The original implementation was re-implemented in Python and used for the comparative experiments.

\textbf{Label Information Bottleneck (LIB)} \cite{zheng2023label}:
LIB is a LE method that aims to reconstruct label distributions using neural networks.
Specifically, it learns a latent representation $\bm{H}$ from input features $\bm{X}$ that retains the most label-relevant information, and then reconstructs the label distribution from $\bm{H}$.
The latent representation is obtained using a neural network consisting of an encoder (Enc) that maps $\bm{X}$ to $\bm{H}$, and a decoder (Dec) that predicts the hard labels from $\bm{H}$.
Separately, a different neural network, referred to as LE-Net, is trained to reconstruct the soft label distribution from $\bm{H}$.
At inference time, the reconstructed label distribution is obtained by feeding $\bm{X}$ into Enc and passing the resulting $\bm{H}$ through LE-Net.
The training objective comprises three components: $\mathcal{L}_{as}$, $\mathcal{L}_{gap}$, and $I(\bm{X}, \bm{H})$.
The term $\mathcal{L}_{as}$ encourages the latent representation to capture more information relevant to the hard-label assignment.
In contrast, $\mathcal{L}_{gap}$ acts as a regularization term, penalizing cases where the reconstructed soft label distribution significantly deviates from the original hard labels.
$I(\bm{X}, \bm{H})$ controls the amount of mutual information shared between $\bm{X}$ and $\bm{H}$ to prevent over-reliance on the input features.
In practice, Enc and Dec are first pretrained, after which LE-Net is trained using the fixed Enc.

The hyperparameters of LIB include the regularization coefficients $\beta_1 \in \mathbb{R}$ and $\beta_2 \in \mathbb{R}$, which control the contribution of $\mathcal{L}_{gap}$ and $I(\bm{X}, \bm{H})$, respectively.
Additional hyperparameters include the hidden layer sizes of Enc, Dec, and LE-Net, as well as training configurations for the pretraining and final training phases.
Zheng et al. (2023) reported that LIB's label reconstruction performance is robust to a wide range of values for $\beta_1$ and $\beta_2$, specifically within $\{0.001, 0.01, \dots, 10\}$ \cite{zheng2023label}.
Accordingly, $\beta_1 = 0.01$ and $\beta_2 = 0.1$ were fixed in our experiments.
The hidden layer size for Enc, Dec, and LE-Net was set to $256$.
For pretraining Enc and Dec, a batch size of $32$, $200$ training epochs, and a learning rate of $0.005$ were used.
For training LE-Net, a batch size of $32$, $100$ epochs, and a learning rate of $0.001$ were employed.
The official implementation of LIB is publicly available at \href{https://github.com/qinghai-zheng/LIBLE/tree/main}{https://github.com/qinghai-zheng/LIBLE/tree/main}.
This implementation was used in the comparative experiments.

\section{Detail of Experimental Result}
\label{apdx:sec_detal_exp_result}

\subsection{Comparison of different parameters in T1OC and T2OC}
\label{apdx:subsec_compare_lambda}

Figures~\ref{fig:compare_lambda_T1OC_image}, \ref{fig:compare_lambda_T2OC_image}, \ref{fig:compare_lambda_T1OC_UCI}, and \ref{fig:compare_lambda_T2OC_UCI} show the results of varying the mixing coefficient $\lambda$ of $p_{\mathrm{A}}(y|\bm{x})$ obtained from T1OC and T2OC on image and UCI datasets.
From Figures~\ref{fig:compare_lambda_T1OC_image} and \ref{fig:compare_lambda_T2OC_image}, it can be observed that, for image datasets, the classification performance remains largely unaffected when $\lambda$ is chosen from the set $\{0.1, 0.15, 0.2, 0.25\}$ for both T1OC and T2OC.
A closer inspection reveals that, for MNIST and fashion-MNIST, the variation in prediction accuracy across different $\lambda$ values becomes more pronounced as the number of training samples decreases.
Figures~\ref{fig:compare_lambda_T1OC_UCI} and \ref{fig:compare_lambda_T2OC_UCI} show that, for UCI datasets, the classification accuracy tends to improve as $\lambda$ increases within the set $\{0.1, 0.15, 0.2, 0.25\}$ for both T1OC and T2OC.
In the method comparison experiments in Section~\ref{sec:experiments}, $\lambda$ was fixed at $0.1$, and the results suggest that further improvements in classification performance may be achievable through appropriate tuning of $\lambda$.

\subsection{Comparison with LE methods}
\label{apdx:subsec_compare_le}

Tables \ref{tab:size_perf_detail_mnist_part1} through \ref{tab:size_perf_detail_pendigits_part2} present the test accuracy of classification models obtained in the experiments described in Section~\ref{subsec:exp_settings_and_result}, for all comparison methods including GLLE and LIB as LE methods.
Each value shown in these tables represents the mean and twice the standard deviation over five trials.
From these results, it is evident that LIB exhibits significantly higher variability in accuracy compared to other methods, with its standard deviation reaching up to 100 times greater in some cases. This indicates that LIB suffers from considerable instability in performance.
Moreover, the classification performance achieved by LIB is generally inferior to that of LS and {\it Hard}, and in some cases, substantially worse.
Due to the instability of LIB, its inclusion in the visualization would hinder interpretability.
Therefore, we omit LIB from the figures and instead include GLLE, which demonstrates relatively more stable performance among the LE-based methods. 
The comparative results among methods, including GLLE, are shown in Figure \ref{fig:compare_size_acc_all}.
As shown in Figure~\ref{fig:compare_size_acc_all}, GLLE also tends to yield classification models with lower overall accuracy.
This is attributed to the difficulty in appropriately defining neighborhood relationships in GLLE.
Specifically, the high dimensionality of the instances and the heterogeneous characteristics of features likely lead to inadequately estimated distances, resulting in suboptimal neighbor selection.
These findings confirm that the introduced methods based on our framework, T1OC and T2OC, outperform both GLLE and LIB among the LE-based approaches.

\begin{figure*}[!h]
    \centering
    \includegraphics[width=\linewidth]{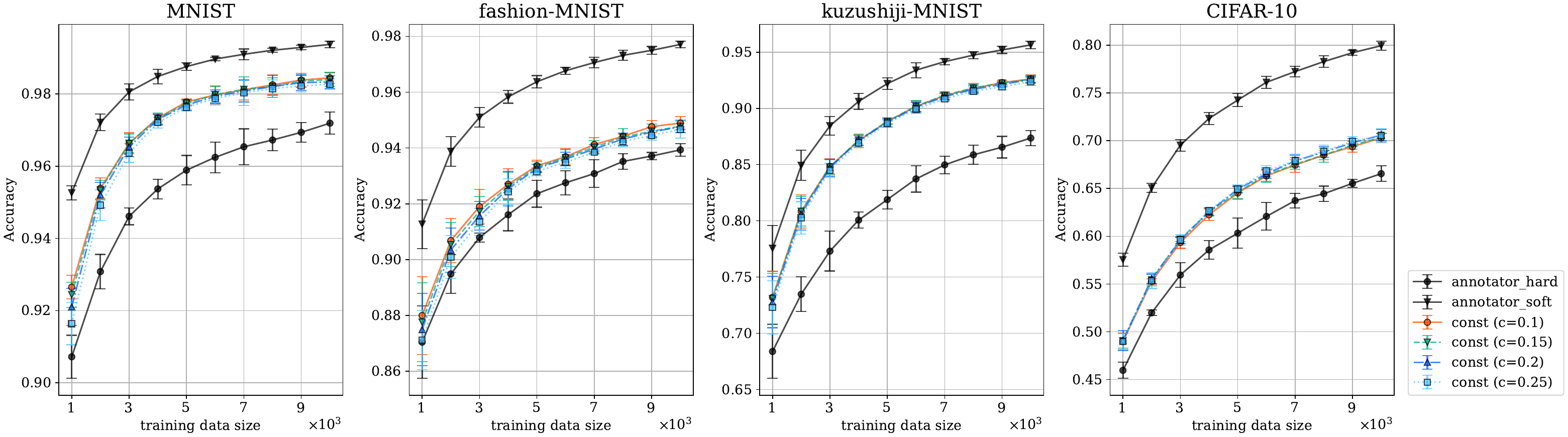}
    \caption{Comparison of different $\lambda$ in T1OC using image datasets.
             The plotted points represent the average of five trials, and the error bars indicate twice the standard deviation.}
    \label{fig:compare_lambda_T1OC_image}
\end{figure*}

\begin{figure*}[!h]
    \centering
    \includegraphics[width=\linewidth]{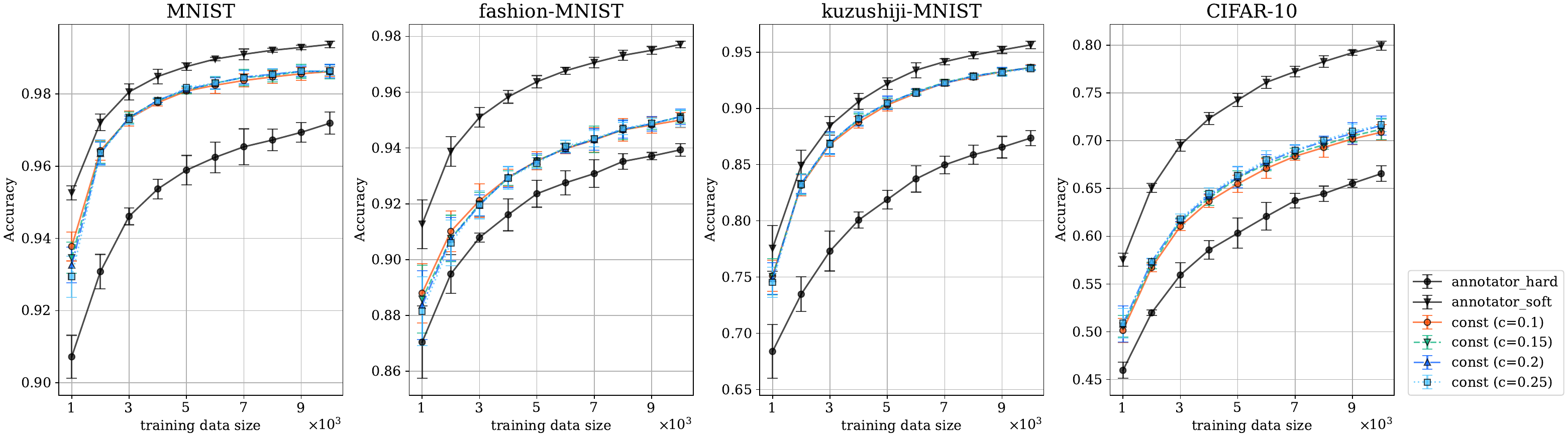}
    \caption{Comparison of different $\lambda$ in T2OC using image datasets.
             The plotted points represent the average of five trials, and the error bars indicate twice the standard deviation.
             }
    \label{fig:compare_lambda_T2OC_image}
\end{figure*}

\begin{figure*}[!h]
    \centering
    \includegraphics[width=\linewidth]{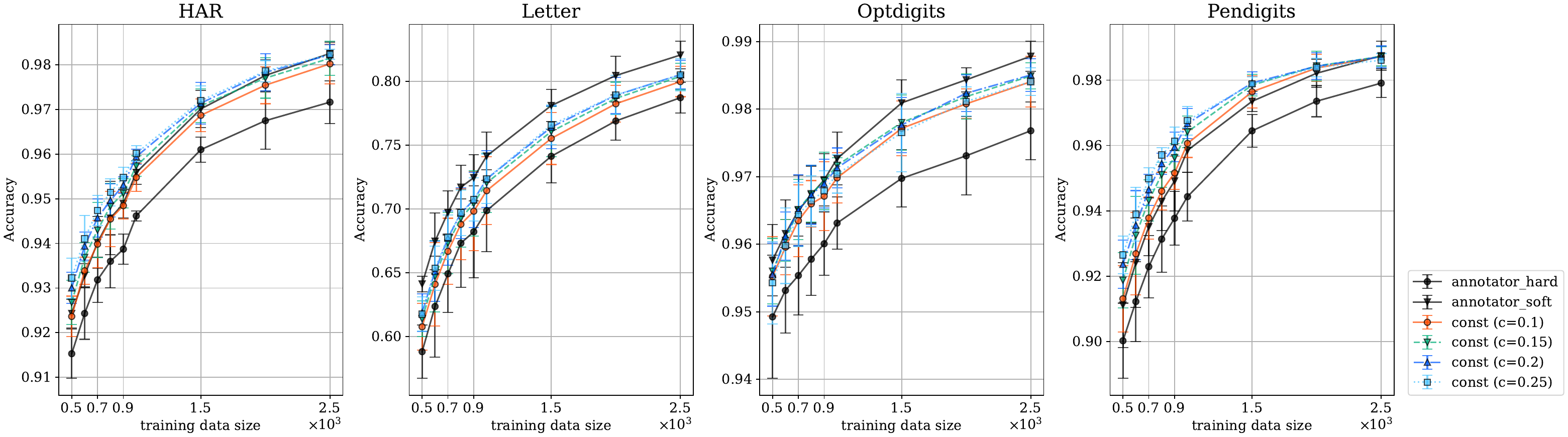}
    \caption{Comparison of different $\lambda$ in T1OC using UCI datasets.
             The plotted points represent the average of five trials, and the error bars indicate twice the standard deviation.
             }
    \label{fig:compare_lambda_T1OC_UCI}
\end{figure*}

\begin{figure*}[!h]
    \centering
    \includegraphics[width=\linewidth]{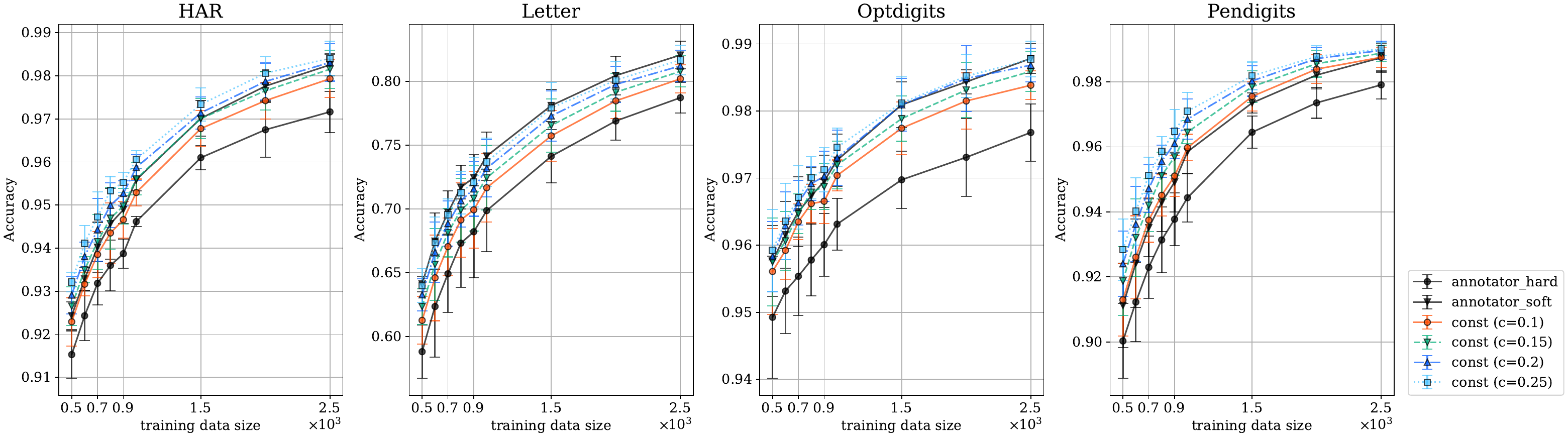}
    \caption{Comparison of different $\lambda$ in T2OC using UCI datasets.
             The plotted points represent the average of five trials, and the error bars indicate twice the standard deviation.
             }
    \label{fig:compare_lambda_T2OC_UCI}
\end{figure*}

\begin{figure*}[!h]
    \centering
    \includegraphics[width=\linewidth]{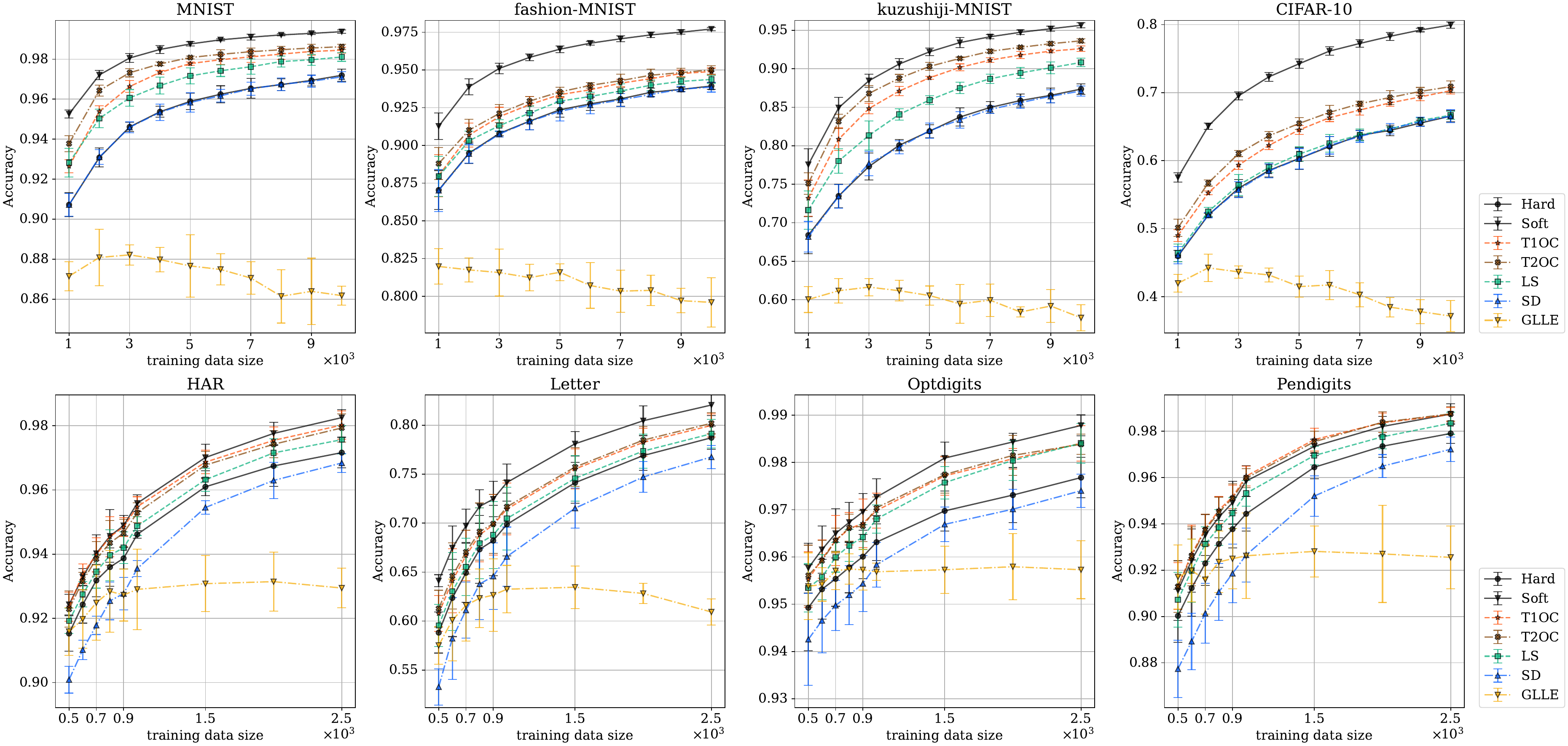}
    \caption{
             Comparison of the classification accuracy of models trained using the labels produced by each method on the image datasets (upper) and UCI datasets (lower).
             This figure includes the result of GLLE as a LE method.
             The plotted points represent the average of five trials, and the error bars indicate twice the standard deviation.}
    \label{fig:compare_size_acc_all}
\end{figure*}

\begin{table}[!h]
    \caption{Detail experimental result using MNIST (Part 1). 
        The first line indicates the number of training data. 
        The score shown is test accuracy. 
        The each value is the mean and twice the standard deviation of five trials. 
        For each number of training data, the value of the method with the highest mean value is bolded. 
}
    \centering
    \begin{tabular}{c|c|c|c|c|c} 
    \hline
     & 1000 & 2000 & 3000 & 4000 & 5000 \\ 
    \hline \hline 
    Hard & 0.907 ±0.006 & 0.931 ±0.005 & 0.946 ±0.002 & 0.954 ±0.003 & 0.959 ±0.004 \\ 
    T1OC & 0.926 ±0.003 & 0.954 ±0.003 & 0.966 ±0.003 & 0.973 ±0.001 & 0.978 ±0.001 \\ 
    T2OC & {\bf 0.938 ±0.004} & {\bf 0.964 ±0.003} & {\bf 0.973 ±0.002} & {\bf 0.978 ±0.001} & {\bf 0.981 ±0.001} \\ 
    LS & 0.928 ±0.007 & 0.950 ±0.005 & 0.961 ±0.004 & 0.967 ±0.004 & 0.972 ±0.004 \\ 
    SD & 0.907 ±0.006 & 0.931 ±0.004 & 0.946 ±0.003 & 0.953 ±0.004 & 0.958 ±0.005 \\ 
    GLLE & 0.871 ±0.007 & 0.881 ±0.014 & 0.882 ±0.005 & 0.880 ±0.006 & 0.877 ±0.016 \\ 
    LIB & 0.603 ±0.791 & 0.478 ±0.789 & 0.769 ±0.433 & 0.945 ±0.081 & 0.970 ±0.014 \\ 
    \hline
    Soft & 0.953 ±0.001 & 0.972 ±0.001 & 0.981 ±0.001 & 0.985 ±0.001 & 0.988 ±0.001 \\ 
    \hline
    \end{tabular}
    \label{tab:size_perf_detail_mnist_part1}
\end{table}

\begin{table}[!h]
    \caption{Detail experimental result using MNIST (Part 2). 
        The first line indicates the number of training data. 
        The score shown is test accuracy. 
        The each value is the mean and twice the standard deviation of five trials. 
        For each number of training data, the value of the method with the highest mean value is bolded. 
}
    \centering
    \begin{tabular}{c|c|c|c|c|c} 
    \hline
     & 6000 & 7000 & 8000 & 9000 & 10000 \\ 
    \hline \hline 
    Hard & 0.962 ±0.004 & 0.965 ±0.005 & 0.967 ±0.003 & 0.969 ±0.003 & 0.972 ±0.003 \\ 
    T1OC & 0.980 ±0.002 & 0.981 ±0.003 & 0.983 ±0.003 & 0.984 ±0.002 & 0.984 ±0.001 \\ 
    T2OC & {\bf 0.982 ±0.002} & {\bf 0.984 ±0.002} & {\bf 0.985 ±0.002} & {\bf 0.986 ±0.002} & {\bf 0.986 ±0.001} \\ 
    LS & 0.974 ±0.003 & 0.976 ±0.004 & 0.979 ±0.003 & 0.980 ±0.003 & 0.981 ±0.002 \\ 
    SD & 0.962 ±0.003 & 0.965 ±0.003 & 0.967 ±0.003 & 0.969 ±0.003 & 0.971 ±0.003 \\ 
    GLLE & 0.875 ±0.008 & 0.871 ±0.008 & 0.861 ±0.013 & 0.864 ±0.017 & 0.862 ±0.005 \\ 
    LIB & 0.849 ±0.496 & 0.782 ±0.685 & 0.976 ±0.016 & 0.977 ±0.010 & 0.978 ±0.011 \\ 
    \hline
    Soft & 0.990 ±0.000 & 0.991 ±0.001 & 0.992 ±0.000 & 0.993 ±0.000 & 0.994 ±0.000 \\ 
    \hline
    \end{tabular}
    \label{tab:size_perf_detail_mnist_part2}
\end{table}

\begin{table}[!h]
    \caption{Detail experimental result using fashion-MNIST (Part 1). 
        The first line indicates the number of training data. 
        The score shown is test accuracy. 
        The each value is the mean and twice the standard deviation of five trials. 
        For each number of training data, the value of the method with the highest mean value is bolded. 
}
    \centering
    \begin{tabular}{c|c|c|c|c|c} 
    \hline
     & 1000 & 2000 & 3000 & 4000 & 5000 \\ 
    \hline \hline 
    Hard & 0.870 ±0.013 & 0.895 ±0.007 & 0.908 ±0.002 & 0.916 ±0.006 & 0.924 ±0.005 \\ 
    T1OC & 0.880 ±0.014 & 0.907 ±0.008 & 0.919 ±0.006 & 0.927 ±0.005 & 0.934 ±0.002 \\ 
    T2OC & {\bf 0.888 ±0.011} & {\bf 0.910 ±0.007} & {\bf 0.921 ±0.006} & {\bf 0.929 ±0.003} & {\bf 0.935 ±0.003} \\ 
    LS & 0.880 ±0.013 & 0.903 ±0.006 & 0.913 ±0.004 & 0.921 ±0.005 & 0.929 ±0.005 \\ 
    SD & 0.870 ±0.014 & 0.894 ±0.006 & 0.907 ±0.002 & 0.916 ±0.006 & 0.923 ±0.006 \\ 
    GLLE & 0.820 ±0.012 & 0.818 ±0.008 & 0.816 ±0.016 & 0.812 ±0.009 & 0.816 ±0.006 \\ 
    LIB & 0.719 ±0.621 & 0.898 ±0.006 & 0.709 ±0.473 & 0.468 ±0.753 & 0.896 ±0.085 \\ 
    \hline
    Soft & 0.913 ±0.004 & 0.939 ±0.003 & 0.951 ±0.002 & 0.958 ±0.001 & 0.964 ±0.001 \\ 
    \hline
    \end{tabular}
    \label{tab:size_perf_detail_fashion_mnist_part1}
\end{table}

\begin{table}[!h]
    \caption{Detail experimental result using fashion-MNIST (Part 2). 
        The first line indicates the number of training data. 
        The score shown is test accuracy. 
        The each value is the mean and twice the standard deviation of five trials. 
        For each number of training data, the value of the method with the highest mean value is bolded. 
}
    \centering
    \begin{tabular}{c|c|c|c|c|c} 
    \hline
     & 6000 & 7000 & 8000 & 9000 & 10000 \\ 
    \hline \hline 
    Hard & 0.928 ±0.004 & 0.931 ±0.005 & 0.935 ±0.003 & 0.937 ±0.001 & 0.939 ±0.002 \\ 
    T1OC & 0.937 ±0.003 & 0.941 ±0.003 & 0.944 ±0.001 & 0.948 ±0.002 & 0.949 ±0.002 \\ 
    T2OC & {\bf 0.940 ±0.002} & {\bf 0.943 ±0.004} & {\bf 0.947 ±0.004} & {\bf 0.948 ±0.003} & {\bf 0.950 ±0.003} \\ 
    LS & 0.932 ±0.003 & 0.936 ±0.005 & 0.940 ±0.003 & 0.942 ±0.002 & 0.944 ±0.003 \\ 
    SD & 0.927 ±0.006 & 0.930 ±0.004 & 0.934 ±0.002 & 0.937 ±0.002 & 0.939 ±0.003 \\ 
    GLLE & 0.807 ±0.015 & 0.803 ±0.014 & 0.804 ±0.010 & 0.797 ±0.008 & 0.796 ±0.016 \\ 
    LIB & 0.925 ±0.023 & 0.936 ±0.006 & 0.939 ±0.003 & 0.796 ±0.580 & 0.943 ±0.004 \\ 
    \hline
    Soft & 0.968 ±0.001 & 0.971 ±0.001 & 0.973 ±0.001 & 0.975 ±0.001 & 0.977 ±0.001 \\ 
    \hline
    \end{tabular}
    \label{tab:size_perf_detail_fashion_mnist_part2}
\end{table}

\begin{table}[!h]
    \caption{Detail experimental result using kuzushiji-MNIST (Part 1). 
        The first line indicates the number of training data. 
        The score shown is test accuracy. 
        The each value is the mean and twice the standard deviation of five trials. 
        For each number of training data, the value of the method with the highest mean value is bolded. 
}
    \centering
    \begin{tabular}{c|c|c|c|c|c} 
    \hline
     & 1000 & 2000 & 3000 & 4000 & 5000 \\ 
    \hline \hline 
    Hard & 0.684 ±0.024 & 0.735 ±0.015 & 0.773 ±0.018 & 0.801 ±0.007 & 0.819 ±0.008 \\ 
    T1OC & 0.732 ±0.023 & 0.808 ±0.015 & 0.848 ±0.007 & 0.871 ±0.006 & 0.888 ±0.002 \\ 
    T2OC & {\bf 0.751 ±0.014} & {\bf 0.832 ±0.009} & {\bf 0.868 ±0.010} & {\bf 0.888 ±0.005} & {\bf 0.903 ±0.005} \\ 
    LS & 0.716 ±0.025 & 0.780 ±0.016 & 0.813 ±0.019 & 0.841 ±0.007 & 0.859 ±0.006 \\ 
    SD & 0.682 ±0.020 & 0.734 ±0.015 & 0.777 ±0.016 & 0.798 ±0.008 & 0.820 ±0.010 \\ 
    GLLE & 0.600 ±0.017 & 0.612 ±0.016 & 0.616 ±0.011 & 0.612 ±0.013 & 0.606 ±0.012 \\ 
    LIB & 0.470 ±0.608 & 0.531 ±0.622 & 0.618 ±0.557 & 0.702 ±0.571 & 0.819 ±0.142 \\ 
    \hline
    Soft & 0.776 ±0.010 & 0.849 ±0.007 & 0.884 ±0.004 & 0.906 ±0.004 & 0.922 ±0.003 \\ 
    \hline
    \end{tabular}
    \label{tab:size_perf_detail_kuzushiji_mnist_part1}
\end{table}

\begin{table}[!h]
    \caption{Detail experimental result using kuzushiji-MNIST (Part 2). 
        The first line indicates the number of training data. 
        The score shown is test accuracy. 
        The each value is the mean and twice the standard deviation of five trials. 
        For each number of training data, the value of the method with the highest mean value is bolded. 
}
    \centering
    \begin{tabular}{c|c|c|c|c|c} 
    \hline
     & 6000 & 7000 & 8000 & 9000 & 10000 \\ 
    \hline \hline 
    Hard & 0.837 ±0.012 & 0.850 ±0.008 & 0.859 ±0.008 & 0.865 ±0.010 & 0.874 ±0.007 \\ 
    T1OC & 0.902 ±0.005 & 0.911 ±0.003 & 0.918 ±0.005 & 0.923 ±0.002 & 0.926 ±0.004 \\ 
    T2OC & {\bf 0.913 ±0.002} & {\bf 0.923 ±0.003} & {\bf 0.928 ±0.002} & {\bf 0.932 ±0.003} & {\bf 0.936 ±0.003} \\ 
    LS & 0.875 ±0.007 & 0.887 ±0.006 & 0.895 ±0.007 & 0.901 ±0.008 & 0.908 ±0.006 \\ 
    SD & 0.834 ±0.011 & 0.847 ±0.005 & 0.856 ±0.008 & 0.864 ±0.008 & 0.871 ±0.007 \\ 
    GLLE & 0.595 ±0.025 & 0.599 ±0.021 & 0.584 ±0.007 & 0.592 ±0.021 & 0.577 ±0.017 \\ 
    LIB & 0.879 ±0.011 & 0.735 ±0.638 & 0.851 ±0.121 & 0.905 ±0.007 & 0.910 ±0.008 \\ 
    \hline
    Soft & 0.934 ±0.003 & 0.941 ±0.001 & 0.947 ±0.002 & 0.952 ±0.002 & 0.956 ±0.002 \\ 
    \hline
    \end{tabular}
    \label{tab:size_perf_detail_kuzushiji_mnist_part2}
\end{table}

\begin{table}[!h]
    \caption{Detail experimental result using CIFAR-10 (Part 1). 
        The first line indicates the number of training data. 
        The score shown is test accuracy. 
        The each value is the mean and twice the standard deviation of five trials. 
        For each number of training data, the value of the method with the highest mean value is bolded. 
}
    \centering
    \begin{tabular}{c|c|c|c|c|c} 
    \hline
     & 1000 & 2000 & 3000 & 4000 & 5000 \\ 
    \hline \hline 
    Hard & 0.460 ±0.008 & 0.520 ±0.003 & 0.559 ±0.013 & 0.586 ±0.010 & 0.603 ±0.016 \\ 
    T1OC & 0.490 ±0.009 & 0.553 ±0.003 & 0.594 ±0.006 & 0.623 ±0.006 & 0.645 ±0.007 \\ 
    T2OC & {\bf 0.501 ±0.013} & {\bf 0.567 ±0.005} & {\bf 0.610 ±0.004} & {\bf 0.637 ±0.006} & {\bf 0.655 ±0.009} \\ 
    LS & 0.464 ±0.013 & 0.525 ±0.006 & 0.564 ±0.016 & 0.590 ±0.007 & 0.610 ±0.011 \\ 
    SD & 0.461 ±0.012 & 0.520 ±0.004 & 0.557 ±0.012 & 0.585 ±0.010 & 0.603 ±0.015 \\ 
    GLLE & 0.420 ±0.013 & 0.442 ±0.020 & 0.436 ±0.009 & 0.432 ±0.010 & 0.415 ±0.015 \\ 
    LIB & 0.173 ±0.117 & 0.236 ±0.173 & 0.239 ±0.046 & 0.230 ±0.122 & 0.191 ±0.145 \\ 
    \hline
    Soft & 0.575 ±0.003 & 0.651 ±0.002 & 0.695 ±0.003 & 0.723 ±0.003 & 0.743 ±0.003 \\ 
    \hline
    \end{tabular}
    \label{tab:size_perf_detail_cifar10_part1}
\end{table}

\begin{table}[!h]
    \caption{Detail experimental result using CIFAR-10 (Part 2). 
        The first line indicates the number of training data. 
        The score shown is test accuracy. 
        The each value is the mean and twice the standard deviation of five trials. 
        For each number of training data, the value of the method with the highest mean value is bolded. 
}
    \centering
    \begin{tabular}{c|c|c|c|c|c} 
    \hline
     & 6000 & 7000 & 8000 & 9000 & 10000 \\ 
    \hline \hline 
    Hard & 0.621 ±0.014 & 0.637 ±0.008 & 0.645 ±0.008 & 0.655 ±0.004 & 0.665 ±0.008 \\ 
    T1OC & 0.663 ±0.006 & 0.675 ±0.008 & 0.685 ±0.005 & 0.694 ±0.006 & 0.703 ±0.005 \\ 
    T2OC & {\bf 0.671 ±0.010} & {\bf 0.684 ±0.004} & {\bf 0.693 ±0.010} & {\bf 0.702 ±0.006} & {\bf 0.709 ±0.008} \\ 
    LS & 0.626 ±0.013 & 0.638 ±0.009 & 0.646 ±0.003 & 0.659 ±0.003 & 0.667 ±0.006 \\ 
    SD & 0.623 ±0.013 & 0.636 ±0.008 & 0.647 ±0.008 & 0.657 ±0.007 & 0.665 ±0.010 \\ 
    GLLE & 0.417 ±0.021 & 0.403 ±0.017 & 0.385 ±0.014 & 0.378 ±0.018 & 0.372 ±0.023 \\ 
    LIB & 0.246 ±0.104 & 0.167 ±0.095 & 0.249 ±0.101 & 0.188 ±0.140 & 0.284 ±0.090 \\ 
    \hline
    Soft & 0.761 ±0.003 & 0.772 ±0.003 & 0.783 ±0.003 & 0.792 ±0.001 & 0.799 ±0.002 \\ 
    \hline
    \end{tabular}
    \label{tab:size_perf_detail_cifar10_part2}
\end{table}

\begin{table}[!h]
    \caption{Detail experimental result using HAR (Part 1). 
        The first line indicates the number of training data. 
        The score shown is test accuracy. 
        The each value is the mean and twice the standard deviation of five trials. 
        For each number of training data, the value of the method with the highest mean value is bolded. 
}
    \centering
    \begin{tabular}{c|c|c|c|c|c} 
    \hline
     & 500 & 600 & 700 & 800 & 900 \\ 
    \hline \hline 
    Hard & 0.915 ±0.006 & 0.924 ±0.006 & 0.932 ±0.005 & 0.936 ±0.006 & 0.939 ±0.003 \\ 
    T1OC & {\bf 0.924 ±0.005} & {\bf 0.934 ±0.003} & {\bf 0.940 ±0.005} & {\bf 0.945 ±0.006} & {\bf 0.948 ±0.003} \\ 
    T2OC & 0.923 ±0.006 & 0.932 ±0.003 & 0.939 ±0.005 & 0.944 ±0.007 & 0.947 ±0.004 \\ 
    LS & 0.919 ±0.004 & 0.927 ±0.004 & 0.935 ±0.005 & 0.940 ±0.008 & 0.942 ±0.005 \\ 
    SD & 0.901 ±0.004 & 0.910 ±0.003 & 0.918 ±0.003 & 0.925 ±0.006 & 0.928 ±0.005 \\ 
    GLLE & 0.916 ±0.007 & 0.920 ±0.009 & 0.925 ±0.012 & 0.928 ±0.013 & 0.927 ±0.008 \\ 
    LIB & 0.850 ±0.337 & 0.656 ±0.704 & 0.801 ±0.504 & 0.811 ±0.570 & 0.689 ±0.507 \\ 
    \hline
    Soft & 0.924 ±0.002 & 0.933 ±0.001 & 0.940 ±0.003 & 0.946 ±0.004 & 0.949 ±0.002 \\ 
    \hline
    \end{tabular}
    \label{tab:size_perf_detail_har_part1}
\end{table}

\begin{table}[!h]
    \caption{Detail experimental result using HAR (Part 2). 
        The first line indicates the number of training data. 
        The score shown is test accuracy. 
        The each value is the mean and twice the standard deviation of five trials. 
        For each number of training data, the value of the method with the highest mean value is bolded. 
}
    \centering
    \begin{tabular}{c|c|c|c|c} 
    \hline
     & 1000 & 1500 & 2000 & 2500 \\ 
    \hline \hline 
    Hard & 0.946 ±0.001 & 0.961 ±0.003 & 0.967 ±0.006 & 0.972 ±0.005 \\ 
    T1OC & {\bf 0.955 ±0.003} & {\bf 0.969 ±0.004} & {\bf 0.975 ±0.004} & {\bf 0.980 ±0.004} \\ 
    T2OC & 0.953 ±0.003 & 0.968 ±0.004 & 0.974 ±0.004 & 0.979 ±0.004 \\ 
    LS & 0.949 ±0.003 & 0.963 ±0.004 & 0.972 ±0.004 & 0.976 ±0.004 \\ 
    SD & 0.936 ±0.003 & 0.955 ±0.002 & 0.963 ±0.006 & 0.968 ±0.003 \\ 
    GLLE & 0.929 ±0.013 & 0.931 ±0.009 & 0.931 ±0.009 & 0.929 ±0.006 \\ 
    LIB & 0.624 ±0.670 & 0.813 ±0.624 & 0.684 ±0.701 & 0.878 ±0.407 \\ 
    \hline
    Soft & 0.956 ±0.001 & 0.970 ±0.002 & 0.978 ±0.002 & 0.983 ±0.001 \\ 
    \hline
    \end{tabular}
    \label{tab:size_perf_detail_har_part2}
\end{table}

\begin{table}[!h]
    \caption{Detail experimental result using Letter (Part 1). 
        The first line indicates the number of training data. 
        The score shown is test accuracy. 
        The each value is the mean and twice the standard deviation of five trials. 
        For each number of training data, the value of the method with the highest mean value is bolded. 
}
    \centering
    \begin{tabular}{c|c|c|c|c|c} 
    \hline
     & 500 & 600 & 700 & 800 & 900 \\ 
    \hline \hline 
    Hard & 0.588 ±0.021 & 0.624 ±0.040 & 0.649 ±0.030 & 0.673 ±0.035 & 0.682 ±0.036 \\ 
    T1OC & 0.608 ±0.018 & 0.641 ±0.033 & 0.667 ±0.026 & 0.688 ±0.028 & 0.698 ±0.031 \\ 
    T2OC & {\bf 0.613 ±0.019} & {\bf 0.646 ±0.034} & {\bf 0.670 ±0.022} & {\bf 0.691 ±0.029} & {\bf 0.699 ±0.030} \\ 
    LS & 0.596 ±0.022 & 0.630 ±0.040 & 0.655 ±0.029 & 0.679 ±0.036 & 0.688 ±0.034 \\ 
    SD & 0.533 ±0.018 & 0.582 ±0.042 & 0.611 ±0.029 & 0.638 ±0.036 & 0.646 ±0.034 \\ 
    GLLE & 0.575 ±0.019 & 0.601 ±0.042 & 0.616 ±0.037 & 0.623 ±0.030 & 0.627 ±0.037 \\ 
    LIB & 0.279 ±0.365 & 0.261 ±0.286 & 0.286 ±0.377 & 0.377 ±0.467 & 0.279 ±0.422 \\ 
    \hline
    Soft & 0.641 ±0.003 & 0.675 ±0.011 & 0.697 ±0.008 & 0.717 ±0.009 & 0.725 ±0.009 \\ 
    \hline
    \end{tabular}
    \label{tab:size_perf_detail_letter_part1}
\end{table}

\begin{table}[!h]
    \caption{Detail experimental result using Letter (Part 2). 
        The first line indicates the number of training data. 
        The score shown is test accuracy. 
        The each value is the mean and twice the standard deviation of five trials. 
        For each number of training data, the value of the method with the highest mean value is bolded. 
}
    \centering
    \begin{tabular}{c|c|c|c|c} 
    \hline
     & 1000 & 1500 & 2000 & 2500 \\ 
    \hline \hline 
    Hard & 0.699 ±0.032 & 0.741 ±0.021 & 0.769 ±0.015 & 0.787 ±0.012 \\ 
    T1OC & 0.714 ±0.026 & 0.755 ±0.020 & 0.782 ±0.014 & 0.800 ±0.012 \\ 
    T2OC & {\bf 0.717 ±0.027} & {\bf 0.757 ±0.020} & {\bf 0.785 ±0.014} & {\bf 0.802 ±0.011} \\ 
    LS & 0.705 ±0.032 & 0.746 ±0.023 & 0.774 ±0.017 & 0.791 ±0.015 \\ 
    SD & 0.666 ±0.030 & 0.715 ±0.020 & 0.747 ±0.016 & 0.768 ±0.012 \\ 
    GLLE & 0.633 ±0.024 & 0.635 ±0.022 & 0.628 ±0.010 & 0.609 ±0.013 \\ 
    LIB & 0.313 ±0.456 & 0.637 ±0.116 & 0.187 ±0.416 & 0.320 ±0.451 \\ 
    \hline
    Soft & 0.741 ±0.009 & 0.781 ±0.006 & 0.804 ±0.008 & 0.821 ±0.005 \\ 
    \hline
    \end{tabular}
    \label{tab:size_perf_detail_letter_part2}
\end{table}

\begin{table}[!h]
    \caption{Detail experimental result using Optdigits (Part 1). 
        The first line indicates the number of training data. 
        The score shown is test accuracy. 
        The each value is the mean and twice the standard deviation of five trials. 
        For each number of training data, the value of the method with the highest mean value is bolded. 
}
    \centering
    \begin{tabular}{c|c|c|c|c|c} 
    \hline
     & 500 & 600 & 700 & 800 & 900 \\ 
    \hline \hline 
    Hard & 0.949 ±0.009 & 0.953 ±0.006 & 0.955 ±0.006 & 0.958 ±0.005 & 0.960 ±0.005 \\ 
    T1OC & 0.955 ±0.006 & {\bf 0.960 ±0.004} & {\bf 0.963 ±0.005} & 0.966 ±0.003 & {\bf 0.967 ±0.005} \\ 
    T2OC & {\bf 0.956 ±0.006} & 0.959 ±0.004 & {\bf 0.963 ±0.003} & {\bf 0.966 ±0.003} & 0.967 ±0.003 \\ 
    LS & 0.953 ±0.005 & 0.956 ±0.005 & 0.960 ±0.003 & 0.962 ±0.003 & 0.964 ±0.003 \\ 
    SD & 0.943 ±0.010 & 0.947 ±0.007 & 0.950 ±0.005 & 0.952 ±0.006 & 0.954 ±0.006 \\ 
    GLLE & 0.954 ±0.007 & 0.954 ±0.004 & 0.957 ±0.004 & 0.957 ±0.003 & 0.957 ±0.004 \\ 
    LIB & 0.789 ±0.690 & 0.793 ±0.694 & 0.821 ±0.601 & 0.681 ±0.721 & 0.862 ±0.458 \\ 
    \hline
    Soft & 0.958 ±0.003 & 0.962 ±0.002 & 0.965 ±0.003 & 0.967 ±0.002 & 0.970 ±0.002 \\ 
    \hline
    \end{tabular}
    \label{tab:size_perf_detail_optical_part1}
\end{table}

\begin{table}[!h]
    \caption{Detail experimental result using Optdigits (Part 2). 
        The first line indicates the number of training data. 
        The score shown is test accuracy. 
        The each value is the mean and twice the standard deviation of five trials. 
        For each number of training data, the value of the method with the highest mean value is bolded. 
}
    \centering
    \begin{tabular}{c|c|c|c|c} 
    \hline
     & 1000 & 1500 & 2000 & 2500 \\ 
    \hline \hline 
    Hard & 0.963 ±0.004 & 0.970 ±0.004 & 0.973 ±0.006 & 0.977 ±0.004 \\ 
    T1OC & 0.970 ±0.004 & 0.977 ±0.004 & 0.981 ±0.002 & {\bf 0.984 ±0.004} \\ 
    T2OC & {\bf 0.970 ±0.002} & {\bf 0.977 ±0.004} & {\bf 0.981 ±0.004} & 0.984 ±0.002 \\ 
    LS & 0.968 ±0.003 & 0.976 ±0.003 & 0.980 ±0.004 & {\bf 0.984 ±0.004} \\ 
    SD & 0.958 ±0.005 & 0.967 ±0.004 & 0.970 ±0.004 & 0.974 ±0.004 \\ 
    GLLE & 0.957 ±0.002 & 0.957 ±0.005 & 0.958 ±0.007 & 0.957 ±0.006 \\ 
    LIB & 0.666 ±0.606 & 0.658 ±0.719 & 0.979 ±0.020 & 0.925 ±0.234 \\ 
    \hline
    Soft & 0.973 ±0.002 & 0.981 ±0.002 & 0.984 ±0.001 & 0.988 ±0.001 \\ 
    \hline
    \end{tabular}
    \label{tab:size_perf_detail_optical_part2}
\end{table}

\begin{table}[!h]
    \caption{Detail experimental result using Pendigits (Part 1). 
        The first line indicates the number of training data. 
        The score shown is test accuracy. 
        The each value is the mean and twice the standard deviation of five trials. 
        For each number of training data, the value of the method with the highest mean value is bolded. 
}
    \centering
    \begin{tabular}{c|c|c|c|c|c} 
    \hline
     & 500 & 600 & 700 & 800 & 900 \\ 
    \hline \hline 
    Hard & 0.900 ±0.011 & 0.912 ±0.012 & 0.923 ±0.010 & 0.931 ±0.010 & 0.938 ±0.008 \\ 
    T1OC & 0.913 ±0.010 & 0.927 ±0.013 & {\bf 0.938 ±0.006} & {\bf 0.946 ±0.006} & {\bf 0.952 ±0.005} \\ 
    T2OC & 0.913 ±0.011 & 0.926 ±0.013 & 0.937 ±0.007 & 0.945 ±0.006 & 0.951 ±0.006 \\ 
    LS & 0.907 ±0.012 & 0.920 ±0.014 & 0.931 ±0.008 & 0.938 ±0.008 & 0.945 ±0.007 \\ 
    SD & 0.877 ±0.012 & 0.889 ±0.012 & 0.901 ±0.013 & 0.911 ±0.012 & 0.919 ±0.013 \\ 
    GLLE & {\bf 0.917 ±0.014} & 0.920 ±0.014 & 0.916 ±0.015 & 0.924 ±0.013 & 0.925 ±0.013 \\ 
    LIB & 0.768 ±0.663 & {\bf 0.943 ±0.013} & 0.799 ±0.394 & 0.746 ±0.435 & 0.726 ±0.579 \\ 
    \hline
    Soft & 0.911 ±0.006 & 0.924 ±0.007 & 0.935 ±0.004 & 0.943 ±0.004 & 0.949 ±0.005 \\ 
    \hline
    \end{tabular}
    \label{tab:size_perf_detail_pendigits_part1}
\end{table}

\begin{table}[!h]
    \caption{Detail experimental result using Pendigits (Part 2). 
        The first line indicates the number of training data. 
        The score shown is test accuracy. 
        The each value is the mean and twice the standard deviation of five trials. 
        For each number of training data, the value of the method with the highest mean value is bolded. 
}
    \centering
    \begin{tabular}{c|c|c|c|c} 
    \hline
     & 1000 & 1500 & 2000 & 2500 \\ 
    \hline \hline 
    Hard & 0.944 ±0.007 & 0.964 ±0.005 & 0.974 ±0.005 & 0.979 ±0.004 \\ 
    T1OC & {\bf 0.961 ±0.004} & {\bf 0.976 ±0.005} & 0.984 ±0.004 & 0.987 ±0.003 \\ 
    T2OC & 0.960 ±0.004 & 0.976 ±0.005 & {\bf 0.984 ±0.004} & {\bf 0.988 ±0.003} \\ 
    LS & 0.953 ±0.006 & 0.970 ±0.004 & 0.978 ±0.005 & 0.983 ±0.005 \\ 
    SD & 0.927 ±0.012 & 0.952 ±0.009 & 0.965 ±0.005 & 0.972 ±0.005 \\ 
    GLLE & 0.926 ±0.018 & 0.928 ±0.011 & 0.927 ±0.021 & 0.926 ±0.014 \\ 
    LIB & 0.915 ±0.141 & 0.732 ±0.654 & 0.950 ±0.093 & 0.826 ±0.616 \\ 
    \hline
    Soft & 0.959 ±0.003 & 0.974 ±0.002 & 0.982 ±0.002 & 0.987 ±0.002 \\ 
    \hline
    \end{tabular}
    \label{tab:size_perf_detail_pendigits_part2}
\end{table}

\clearpage

\section{Limitation}
\label{apdx_sec:limitation}

The framework introduced in Section~\ref{sec:formulation}, along with the theoretical analysis in Section~\ref{subsec:interpret_formulation}, assumes a fixed set of labeled classes, i.e., no unseen classes exist and the emergence of new classes is not considered.
As such, the current framework does not address scenarios where classes are continuously added over time.
Extending our framework and theoretical analysis to accommodate the presence of unknown or emerging classes is an important future direction, and our results are considered to provide a solid foundation for such extensions.

The generalization error analysis conducted in Section~\ref{subsec:error_analysis} assumes that samples are independently and identically distributed (i.i.d.).
Therefore, the results in Section~\ref{subsec:error_analysis} are not directly applicable to cases where samples exhibit temporal or other dependencies.
Nevertheless, the theoretical results we establish here are expected to serve as a useful starting point when designing proofs under relaxed assumptions that allow for such dependencies.

Furthermore, the analysis in Section~\ref{subsec:error_analysis} assumes that the loss function is Lipschitz continuous and bounded.
Many loss functions used in practice satisfy these assumptions.
Although certain loss functions such as the cross-entropy loss do not strictly meet these assumptions, they are often effectively bounded in real-world applications, due to the use of finite prediction values and the addition of small constants to prevent logarithmic divergence.
Thus, this assumption does not present a practical limitation in most usage scenarios.

\section{About Code of Ethics and Broader Impacts}
\label{apdx_sec:code_of_ethics_broader_impacts}

\textbf{About Potential Harms.}
This study does not involve human subjects, and all datasets used are publicly available.

\textbf{Potential Positive Societal Impacts.}
The findings of this paper offer a new direction, namely the use of additional supervision, as a viable alternative for improving classification accuracy in real-world applications where increasing the number of labeled data points is infeasible.
Such constraints on data acquisition are common across a wide range of fields, including healthcare and industrial applications.
In the healthcare domain, for instance, acquiring additional instances can be particularly difficult in tasks such as rare disease classification, where patient samples are inherently limited.
Moreover, when the input features involve long-term pathological changes, collecting more data can require extended periods of observation, making rapid dataset expansion impractical.
In industrial settings, the cost of collecting even a single additional instance point may be prohibitively high.
Our results can therefore contribute to improving model performance in these practically constrained scenarios.

\textbf{Potential Negative Societal Impacts.}
Because our framework allows for the integration of additional supervision into datasets with hard labels, there is a potential risk that this additional supervision could be misused to construct biased or malicious classification models.
For example, one might intentionally encode biased predictions with respect to specific attributes of personal information included in the input features.
As a precaution in deployment scenarios, we recommend that model developers obtain additional supervision only from trusted sources and rigorously validate its content before use.

\end{document}